\documentclass[11pt]{article}
\usepackage[margin=1.0in]{geometry}

\usepackage{graphicx}
\usepackage{amsmath,amssymb,amsfonts,amsthm,mathrsfs}
\usepackage{color}
\usepackage[small,bf]{caption}
\usepackage{makecell}
\newcommand{\mylabel}[2]{#2\def\@currentlabel{#2}\label{#1}}
\makeatother
\definecolor{dgreen}{rgb}{0,0.5,0}
\usepackage[colorlinks=true,linkcolor=blue,citecolor=dgreen]{hyperref}
\usepackage{cleveref}
\usepackage{cases}
\usepackage{xspace}
\usepackage{adjustbox}
\usepackage[table]{xcolor}
\usepackage{multirow}
\usepackage{subcaption}
\usepackage{arydshln}
\usepackage{algorithm}
\usepackage[noend]{algorithmic}
\usepackage{wrapfig}
\usepackage{enumitem}
\usepackage[normalem]{ulem}
\setlist{nosep}
\allowdisplaybreaks

\newif\ifneurips
\neuripsfalse
\newif\ifshuffle
\shuffletrue

\newcommand{\nc}{\newcommand}
\nc{\DMO}{\DeclareMathOperator}
\nc\todo[1]{\textcolor{red}{[TODO: #1]}}
\nc{\st}{\star}
\nc\m[2]{m_{#1}(#2)}
\nc{\ReduceTree}{\texttt{ReduceTree}\xspace}
\nc{\PolyPriLearn}{\texttt{PolyPriLearn}\xspace}
\nc{\PPPLearn}{\texttt{PolyPriPropLearn}\xspace}
\nc{\GenericLearner}{\texttt{GenericLearner}\xspace}
\nc{\BR}{\mathbb{R}}
\nc{\BA}{\mathsf{A}}
\nc{\BM}{\mathbb{M}}
\nc{\BT}{\mathbb{T}}
\nc{\BN}{\mathbb{N}}
\nc{\BZ}{\mathbb{Z}}
\nc{\ep}{\varepsilon}
\DMO{\height}{ht}
\renewcommand{\epsilon}{\varepsilon}
\nc{\ra}{\rightarrow}
\DMO{\Err}{err}
\DMO{\Opt}{opt}
\DMO{\Est}{Est}
\DMO{\good}{good}
\DMO{\negpt}{neg-pt}
\DMO{\VV}{{V}}
\DMO{\LL}{{L}}
\DMO{\finsupp}{fin}
\DMO{\supp}{supp}
\nc{\fin}{{\finsupp}}
\nc{\err}[2]{\Err_{#1}(#2)}
\nc{\rca}{\mathscr{B}}
\nc{\bt}{b}
\nc{\hMLp}{\hat\ML'}

\nc{\Rprot}{R}
\nc{\Sprot}{S}
\nc{\Aprot}{A}
\nc{\Pprot}{P}
\nc{\Pdist}{D}
\nc{\Qdist}{F}
\nc{\pdist}{d}
\nc{\qdist}{f}

\nc{\DP}{differentially private\xspace}

\nc{\SD}{\mathscr{D}}
\nc{\la}{\lambda}
\DMO{\KL}{KL}
\DMO{\Unif}{Unif}
\nc{\nn}{\varnothing}
\DMO{\SOA}{SOA}
\nc{\soa}[2]{\SOA_{#1}(#2)}
\nc{\soaf}[1]{\SOA_{#1}}
\nc{\gRes}[2]{\hat\MG({#1},{#2})}
\nc{\emp}{\hat P_{S_n}}
\DMO{\Red}{red}
\DMO{\Irred}{irred}
\nc{\Ired}{I^{\Red}}
\nc{\Iirred}{I^{\Irred}}

\nc{\CS}{\mathsf{CS}}
\nc{\TV}{\mathrm{tv}}

\DMO{\ssmp}{ssmp}
\DMO{\agg}{agg}
\DMO{\final}{final}

\nc{\pp}{p}
\nc{\PP}{P}
\nc{\QQ}{Q}
\nc{\DD}{D}

\DMO{\RAPPOR}{{RAPPOR}}
\nc{\RAP}{\RAPPOR}
\DMO{\RR}{RR}
\nc{\MD}{\mathcal{D}}
\nc{\ML}{\mathcal{L}}
\nc{\di}{P}
\nc{\MO}{\mathcal{O}}
\nc{\MM}{\mathcal{M}}
\nc{\MZ}{\mathcal{Z}}
\nc{\MU}{\mathcal{U}}
\nc{\MP}{\mathcal{P}}
\nc{\MQ}{\mathcal{Q}}
\nc{\poly}{\mathrm{poly}}
\DMO{\treesum}{TreeSum}
\DMO{\lapsum}{LapSum}
\DMO{\checksum}{CheckSum}
\nc{\MDts}{\MD_{\treesum}}
\nc{\MDls}{\MD_{\lapsum}}
\nc{\MDcs}{\MD_{\checksum}}
\nc{\MC}{\mathcal{C}}
\nc{\MT}{\mathcal{T}}
\nc{\MS}{\mathcal{S}}
\nc{\MX}{\mathcal{X}}
\nc{\MY}{\mathcal{Y}}
\nc{\MA}{\mathcal{A}}
\nc{\MB}{\mathcal{B}}
\nc{\MJ}{\mathcal{J}}
\nc{\MF}{\mathcal{F}}
\nc{\MG}{\mathcal{G}}
\nc{\MR}{\mathcal{R}}
\nc{\p}{\Pr}
\nc{\E}{\mathbb{E}}
\nc{\tablesize}{s}
\DMO{\Hist}{hist}
\DMO{\Reg}{Reg}
\DMO{\prdim}{PRDim}
\nc{\eps}{\epsilon}
\nc{\hist}{\mathrm{hist}}

\nc{\ba}{\mathbf{a}}
\nc{\bx}{\mathbf{x}}
\nc{\bs}{\mathbf{s}}
\nc{\bv}{\mathbf{v}}
\nc{\bw}{\mathbf{w}}
\nc{\by}{\mathbf{y}}
\nc{\bz}{\mathbf{z}}
\nc{\ind}{\mathbf{1}}

\DMO{\sr}{sr}
\DMO{\Med}{Med}
\DMO{\Ber}{Ber}
\DMO{\Bin}{Bin}
\DMO{\Had}{Had}

\nc{\ME}{\mathcal{E}}
\DMO{\View}{View}
\nc{\B}{B}
\nc{\M}{M}
\nc{\ha}{\kappa}
\nc{\hk}{k}
\DMO{\pre}{pre}
\nc{\MH}{\mathcal{H}}
\DMO{\Ldim}{LDim}
\DMO{\SQdim}{SQDim}
\DMO{\Tdim}{Tdim}
\DMO{\sfat}{sfat}
\DMO{\fat}{fat}
\DMO{\vc}{VCdim}
\DMO{\FO}{FO}
\DMO{\CM}{CM}
\DMO{\NB}{NB}
\DMO{\hb}{\beta}
\nc{\MW}{\mathcal{W}}
\nc{\MV}{\mathcal{V}}
\nc{\MK}{\mathcal{K}}
\nc{\MN}{\mathcal{N}}

\nc{\BB}{\{0,1\}}
\nc{\bW}{\mathbf{W}}
\nc{\eell}{\ell}
\nc{\EELL}{L}
\nc{\q}{q}
\DMO{\size}{size}
\nc{\ts}{\tilde{s}}
\nc{\tc}{\tilde{c}}
\nc{\bone}{\mathbf{1}}
\DMO{\stat}{STAT}
\DMO{\TIME}{\mathsf{time}}
\nc{\hu}{\hat{u}}

\DeclareMathOperator*{\argmax}{arg\,max}

\makeatletter
\newtheorem*{rep@theorem}{\rep@title}
\newcommand{\newreptheorem}[2]{%
\newenvironment{rep#1}[1]{%
 \def\rep@title{#2~\ref{##1}}%
 \begin{rep@theorem}}%
 {\end{rep@theorem}}}
\makeatother

\newtheorem{theorem}{Theorem} %[section]
\newtheorem{corollary}[theorem]{Corollary}

\newtheorem{lemma}[theorem]{Lemma}
\newtheorem{informal theorem}[theorem]{Informal Theorem}

\newtheorem{defn}[theorem]{Definition}

\usepackage{xcolor}
\usepackage{cancel}

\newcount\Comments
\newcommand{\badih}[1]{\ifnum\Comments=1\textcolor{red}{[Badih: #1]}\fi}
\newcommand{\pasin}[1]{\ifnum\Comments=1\textcolor{red}{[Pasin: #1]}\fi}
\newcommand{\ravi}[1]{\ifnum\Comments=1\textcolor{cyan}{[Ravi: #1]}\fi}

\title{User-Level Private Learning via Correlated Sampling}

\author{
Badih Ghazi
\hspace*{1cm}
Ravi Kumar
\hspace*{1cm}
Pasin Manurangsi \\
Google, Mountain View, CA. \\
\texttt{\footnotesize \{badihghazi, ravi.k53\}@gmail.com, pasin@google.com}
}

\date{\today}

\begin{document}

\maketitle

\begin{abstract}
Most works in learning with differential privacy (DP) have focused on the setting where each user has a single sample. In this work, we consider the setting where each user holds $m$ samples and the privacy protection is enforced at the level of each user's data.  We show that, in this setting, we may learn with a much fewer number of users. Specifically, we show that, as long as each user receives sufficiently many samples, we can learn any privately learnable class via an $(\eps, \delta)$-DP algorithm using only $O(\log(1/\delta)/\eps)$ users. For $\eps$-DP algorithms, we show that we can learn using only $O_{\eps}(d)$ users even in the local model, where $d$ is the probabilistic representation dimension. In both cases, we show a nearly-matching lower bound on the number of users required.

A crucial component of our results is a generalization of \emph{global stability}~\cite{BunLM20} that allows the use of public randomness. Under this relaxed notion, we employ a correlated sampling strategy to show that the global stability can be boosted to be arbitrarily close to one, at a polynomial expense in the number of samples. 
\end{abstract}

\section{Introduction}

Differential privacy (DP)~\cite{DworkMNS06, dwork2006our} has emerged as the accepted notion for quantifying the privacy of algorithms, whereby a method is considered private if the presence or absence of a single user has a negligible impact on its output. The two most widely-studied models of DP are the \emph{central} model, where an analyzer has access to the raw user data, and is required to output a private answer, and the \emph{local} model \cite{warner1965randomized, kasiviswanathan2011can, duchi2013local}, where the output of each user is required to be private.  DP has become a widely adopted standard in both industry \cite{erlingsson2014rappor,CNET2014Google, greenberg2016apple,dp2017learning, ding2017collecting} and government agencies, including the recent 2020 US Census \cite{abowd2018us}. For a technical overview of DP, see the monographs by Dwork and Roth~\cite{DworkR14} and by Vadhan~\cite{Vadhan17}.  

DP has gained spotlight in machine learning (e.g.,~\cite{chaudhuri2011differentially,abadi2016deep}), with an increased emphasis on protecting the privacy of user data used for training models.  In the traditional notion of DP, the goal is to protect the privacy of each training example, where it is  assumed that each user contributed precisely one such example; this is sometimes referred to as \emph{item-level} privacy.  However, a more realistic and practical setting is where a single user can contribute more than one training example.   Here, the goal would be so-called \emph{user-level} privacy, i.e., protecting the privacy of \emph{all} the training examples contributed by a single user.  This is especially relevant for federated learning settings, where each user can contribute multiple training examples~\cite{mcmahan2017learning, wang2019beyond, augenstein2019generative, epasto2020smoothly}; see the survey by Kairouz et al.~\cite[Section 4.3.2]{kairouz2019advances}, where the question of determining trade-offs between item-level and user-level DP is highlighted.  It then becomes important to understand the learnability implications of this distinction between user-level vs item-level privacy. 

One way to understand this problem is to artificially limit the number of training examples contributed by each user.  This has been explored for some analytics and learning tasks in~\cite{amin2019bounding, wilson2019differentially} and is related to node-level DP \cite{kasiviswanathan2013analyzing}.  While this is an interesting line of work, it does not sufficiently address the core of the problem.  For instance, is learning possible with only a small number of users if each user contributes sufficiently many training examples?  This question was addressed by Liu et al.~\cite{LiuSYK020} for learning discrete distributions and by Levy et al.~\cite{LSAKKMS21} for some learning tasks including mean estimation, ERM with smooth losses, stochastic convex optimization.  They show that the privacy cost decreases faster as the number of samples per user increases.

In this paper we address the question in a very general setting: what can user-level privacy gain for \emph{any} privately PAC learnable class?
Recall that it had recently been shown that a class is learnable via $(\eps, \delta)$-DP algorithms iff it is online learnable~\cite{AlonLMM19,BunLM20}, which is in turn equivalent to the class having a finite \emph{Littlestone dimension}~\cite{littlestone_learning_1987}. Furthermore, it is also known that a class is learnable via $\eps$-DP algorithms iff it has a finite \emph{probabilistic dimension}~\cite{BeimelNS19}.

%In this paper we address the question in a very general setting: what can user-level privacy gain for \emph{any} online learnable class? Recall that a class is online learnable if and only if it has finite Littlestone dimension, which was recently shown to be equivalent to the DP learnability of the class \cite{BunLM20}.

As discussed below, our protocols are based on a novel connection between \emph{correlated sampling}---a tool from sketching and approximation algorithms~\cite{Broder97,KleinbergT02,charikar2002similarity}---and DP learning.
% and the study of parallel repetition \cite{holenstein2007parallel,rao2011parallel,barak2008rounding}, and

\subsection{Our Results}

Our first main result is that, for any online learnable class, it is possible to learn the class with an $(\eps, \delta)$-DP algorithm using only $O(\log(1/\delta) / \eps)$ users, as long as each user has at least $\poly(d/\alpha)$ samples (\Cref{thm:apx-dp-learner-generic}), where $d$ is the Littlestone dimension and $\alpha$ is the error of the hypothesis  output by the learner. It should be noted that the remarkable and arguably surprising aspect of this result is that we can learn using a \emph{constant} number of samples (depending only on the privacy parameters $\eps, \delta$), regardless of how complicated the class might be, as long as the class is online learnable.
(For all formal definitions, see \Cref{sec:prelim}.) 
Indeed, previous work~\cite{amin2019bounding} had explicitly conjectured that such a bound is impossible for user-level learning, albeit in a different setting than ours.
\begin{theorem} \label{thm:apx-dp-learner-generic}
Let $\alpha, \beta \in (0, 0.1)$, and $C$ be any concept class with finite $\Ldim(C) = d$. Then, for any $\eps, \delta \in (0, 1)$, there exists an $(\eps, \delta)$-DP  $(\alpha, \beta)$-accurate learner for $C$ that requires $O\left(\log(1/(\delta \beta)) / \eps\right)$ users where each user has $\tilde{O}_{\beta}\left((d/\alpha)^{O(1)}\right)$ samples.
\end{theorem}

\ifshuffle
Our algorithm in \Cref{thm:apx-dp-learner-generic} can in fact be extended to work even in the weaker \emph{shuffle} model of DP. We provide more detail about such an extension in \Cref{app:shuffle}.
\fi

Our second result is a generic $\eps$-DP learner in the local model for any class $C$ with finite probabilistic representation dimension $\prdim(C)$. This gives a separation in the local DP model between the user-level and item-level settings; the sample complexity in the latter is known to be polynomial in the statistical query (SQ) dimension~\cite{kasiviswanathan2011can}, which can be exponentially larger than the probabilistic representation dimension.  A simple example of such a separation is PARITY on $d$ unknowns, whose probabilistic representation dimension is $O(d)$ (this can be seen by taking the class itself to be its own representation), whereas its SQ dimension is $2^{\Omega(d)}$~\cite{BlumFJKMR94}.  In this case, our \Cref{thm:local-dp-learner-generic} implies that PARITY can be learned using only $O_{\beta}(d/\eps^2)$ users when each user has $O_{\alpha, \beta}(d^3)$ examples; by contrast, the aforementioned lower bound~\cite{kasiviswanathan2011can} implies that, in the item-level setting (where each user has a single example), $2^{\Omega(d)}$ users are needed.

\begin{theorem} \label{thm:local-dp-learner-generic}
Let $\alpha, \beta \in (0, 0.1)$, and $C$ be any concept class with finite $\prdim(C) = d$. Then, for any $\eps \in (0, 1)$, there exists an $\eps$-DP  $(\alpha, \beta)$-accurate learner for $C$ in the (public randomness) local model that requires $O\left(\frac{d + \log(1/\beta)}{\eps^2}\right)$ users where each user has $\tilde{O}_{\beta}\left(d^3/\alpha^2\right)$ samples.
\end{theorem}
In the central model, we get a slightly improved bound on the number of users in terms of $1/\eps$.
\begin{theorem} \label{thm:pure-dp-learner-generic}
Let $\alpha, \beta \in (0, 0.1)$, and $C$ be any concept class with finite $\prdim(C) = d$. Then, for any $\eps \in (0, 1)$, there exists an $\eps$-DP  $(\alpha, \beta)$-accurate learner for $C$ that requires $O\left(\frac{d + \log(1/\beta)}{\eps}\right)$ users where each user has $\tilde{O}_{\beta}\left(d^3/\alpha^2\right)$ samples.
\end{theorem}

Interestingly, we can also show that the number of users required in the above results is essentially the smallest possible (up to a $1/\eps$ factor in \Cref{thm:local-dp-learner-generic}), as stated below.

\begin{lemma}\label{thm:lb-main}
For any $\alpha, \beta \leq 1/4, \eps \in (0, 1)$ and $\delta \in (0, 0.1\eps^{1.1})$, if there exists an $(\eps, \delta)$-DP $(\alpha, \beta)$-accurate learner on $n$ users for a concept class $C$, then we must have $n \geq \Omega(\min\{\log(1/\delta), \prdim(C)\} / \eps)$.
\end{lemma}

The summary of our results described above can be found in~\Cref{table:summary}.

While our previous results establish nearly tight bounds on the number of users required for DP learning, they are in general not efficient.  For example, the pure-DP learners have running times that grow (at least) exponentially in the size of the probabilistic representation, and the approximate-DP learner similarly has a running time that grows (at least) exponentially in the Littlestone dimension. Our final result investigates how to get efficient learners.  Informally, we show that we can take any efficient SQ algorithm and turn it into an efficient learner in the user-level setting (\Cref{thm:sq-reduction}). 

\begin{theorem} \label{thm:sq-reduction}
Let $C$ be any concept class, and suppose that there exists an algorithm $\BA$ that can $\alpha$-learn $C$ using $q$ statistical queries  $\stat_{\MD}(\tau)$.  Furthermore, suppose that any hypothesis output by $\BA$ can be represented by $b$ bits. Then, there exist the following algorithms:
\\
(i) An $(\eps, \delta)$-DP $(\alpha, \beta)$-accurate learner with $O(\log(1/\delta)/\eps)$ users, where each user has $\poly\left(\frac{q}{\beta\tau}\right)$ samples.
\\
(ii) An $\eps$-DP $(\alpha, \beta)$-accurate learner with $O\left(\frac{b + \log(1/\beta)}{\eps}\right)$ users, where each user has $\poly\left(\frac{q}{\beta\tau}\right)$ samples.
\\
(iii) An $\eps$-DP $(\alpha, \beta)$-accurate learner in the (public randomness) local DP model with $O\left(\frac{b + \log(1/\beta)}{\eps^2}\right)$ users, where each user has $\poly\left(\frac{q}{\beta\tau}\right)$ samples.

Moreover, all DP learners described above run in time $\poly(\TIME(\BA), 1/ \tau, 1 / \beta)$.
\end{theorem}

Thanks to the abundance of SQ learning algorithms, the above result can be applied to turn those into efficient user-level DP learners.  We discuss some interesting examples of these in \Cref{sec:efficient-sq-examples}.

\begin{table}[t]
    \centering
    \footnotesize
%\begin{adjustbox}{center}
\begin{tabular}{|c|c|c|c|c|c|}
    %   \hline
    \cline{3-6}
      \multicolumn{2}{c|}{} & \multicolumn{3}{c|}{User-Level} & Item-Level\\
    %   \hline
    \cline{3-6}
          \multicolumn{1}{c}{} &
          \multicolumn{1}{c|}{Bounds} & \# users & \# samples/user & Ref. & \# users\\ 
         \hline
         $\eps$-DP & Upper & $O(\prdim / \eps)$ & $\prdim^{O(1)}$ & \Cref{thm:pure-dp-learner-generic} & $\Theta(\prdim / \eps)$ \\ 
         \cdashline{2-5}
         (Central) & Lower & $\Omega(\prdim / \eps)$ & - & \Cref{thm:lb-main} & \cite{BeimelNS19} \\
      \hline
      $\eps$-DP & Upper & $O(\prdim / \eps^2)$ & $\prdim^{O(1)}$ & \Cref{thm:local-dp-learner-generic} & $\SQdim^{\Theta(1)}$ \\ 
         \cdashline{2-5}
         (Local) & Lower & $\Omega(\prdim / \eps)$ & - & \Cref{thm:lb-main} & \cite{kasiviswanathan2011can} \\
      \hline
      $(\eps, \delta)$-DP & Upper & $O(\log(1/\delta)/\eps)$ & $\Ldim^{O(1)}$ & \Cref{thm:apx-dp-learner-generic} & $\Ldim^{O(1)}$~\cite{GGKM20} \\ 
      \cdashline{2-6}
      (Central) & Lower & $\Omega(\min\{\log(1/\delta),$ & - & \Cref{thm:lb-main} & $\Omega(\log^* \Ldim)$ \\
      & & $\qquad \prdim\}/\epsilon)$ & & & \cite{AlonLMM19}\\
    \hline
    \end{tabular}
%    \end{adjustbox}
     \caption{Summary of our results in the user-level setting and prior results in the item-level setting. For simplicity, we assume that the accuracy parameter and success probability of the learner are constants, and we disregard their dependencies. Our lower bounds hold regardless of the number of samples each user receives.
    \label{table:summary}}
  \end{table}

\subsection*{Independent Work of Impagliazzo et al.~\cite{reproducibility}}

As an intermediate step of our proofs, we define a property called \emph{pseudo-globally stability} for learning algorithms (\Cref{defn:pseudo-stable}) and provide several such algorithms (\Cref{cor:stability-from-littlestone,cor:stability-from-representation}, and \Cref{lem:sq-reduction}). In an independent work, Impagliazzo et al.~\cite{reproducibility} studies a similar notion under the name \emph{reproducibility} and provide several reproducible algorithms e.g. for heavy hitters, SQ-based algorithms and learning halfspaces. Below we provide a more detailed discussion on the similarities and differences between the two papers:
\begin{itemize}
\item \textbf{Definition.} Strictly speaking, the main definition in~\cite{reproducibility} is slightly different compared to ours, but they note in the appendix that the two definitions are equivalent up to a polynomial factor in the parameters.
\item \textbf{SQ Algorithms.} Both Impagliazzo et al.'s work and ours (\Cref{lem:sq-reduction}) give generic reductions for turning SQ algorithms to pseudo-globally stable ones.
\item \textbf{Amplification of Stability Parameter.} In~\cite[Theorem A.2]{reproducibility}, a reduction for decreasing the stability parameter is given. Indeed, one can also view our reduction in \Cref{thm:stability-to-pseudo} in this form but our result is weaker as our reduction starts out with a (not pseudo) globally stable algorithms, whereas their reduction works even when starting with pseudo globally stable algorithms.
\item \textbf{Heavy Hitter Algorithms.} Our aforementioned reduction also implicitly gives an algorithm for heavy hitters. Once again, this is weaker than that in Impagliazzo et al.: ours only gives a single heavy hitter whereas that of~\cite{reproducibility} can provide a list of all heavy hitters.
\item \textbf{Additional Results in~\cite{reproducibility}.} ~\cite{reproducibility} also contains many additional results, including a lower bound on the overhead due to pseudo-global stability, a pseudo-globally stable median algorithm and a pseudo-globally stable algorithm for learning halfspaces with a margin.
\end{itemize}

\subsection{Proof Overview}

% \badih{Should we mention somewhere that simply applying the group property of DP yields to a large required number of users?}

For simplicity of presentation, we will focus on \Cref{thm:apx-dp-learner-generic}; we will briefly discuss the proofs of the other results, which are similar in flavor, at the end of this section.

Let us assume for the moment that each user, given their $m$ samples drawn i.i.d. from $\MD$, can output the same hypothesis $h^*$ (with small error) with high probability. If this holds, then we would be done: we can simply run a DP selection\footnote{\Cref{sec:dp-tools} contains the formal definition of the selection problem and known DP algorithms for it.} algorithm to pick the most frequently seen hypothesis. 

This assumption is quite strong, but not completely unreasonable. Specifically, Bun et al.~\cite{BunLM20}---in their seminal work that characterizes hypothesis classes learnable in the item-level DP setting---showed that it is possible to come up with a learner that outputs some hypothesis $h^*$ with probability $2^{-O(d)}$, where $d$ denotes the Littlestone dimension of the concept class. We may attempt to use this in the approach described above, but this does not work: in order to even see $h^*$ at all (with say a constant probability), we would need $2^{O(d)}$ users, which is prohibitive!

To overcome this, we exploit shared randomness between users. Our main technical result here is that, if the users share randomness, we can ensure that they output the same $h^*$ with probability arbitrarily close to one; we can then run the DP selection algorithm to pick $h^*$.  This immediately allows our overall strategy described above to go through.

The shared randomness is used in our algorithm(s) via \emph{correlated sampling}. Recall that a correlated sampling strategy is an algorithm that takes in a probability distribution $\MP$ together with randomness $r$. The guarantee is that, if we run it on two  distributions $\MP_1, \MP_2$ but with the same randomness $r$, then the probability (over $r$) that the outputs disagree is at most a constant times the total variation distance between $\MP_1$ and $\MP_2$.  The task then is simply to compute, for each user $i$, such a probability distribution $\MP_i$ from their own samples such that the $\MP_i$'s do not differ much between different users.  
%
\iffalse
In summary, the outline of our algorithm is presented in \Cref{alg:dp-learner-outline}.

\begin{algorithm}
\caption{User-Level Learner}
\label{alg:dp-learner-outline}
\begin{algorithmic}
\STATE \textbf{Input: } samples $\{(x^i_1, y^i_1), \dots, (x^i_m, y^i_m)\}_{i \in [n]}$, public randomness $r$
\FOR{Each user $i \in [n]$}
\STATE Construct a distribution $P_i$ on output hypotheses from samples $(x^i_1, y^i_1), \dots, (x^i_m, y^i_m)$.
\STATE $h_i \leftarrow \text{CorrelatedSampling}(P, r)$.
\ENDFOR
\RETURN \text{DPSelection}$(h_1, \dots, h_n)$.
\end{algorithmic}
\end{algorithm}
\fi

Our algorithms in \Cref{thm:apx-dp-learner-generic,thm:local-dp-learner-generic,thm:pure-dp-learner-generic} follow this framework. The differences are in the DP selection algorithms (based on central vs local model and whether we are interested in pure- or approximate-DP)---and, more importantly---how we construct the distribution $\MP_i$ of hypotheses.  In the case of the approximate-DP learner (\Cref{thm:apx-dp-learner-generic}), we build this on top of a learning algorithm of Ghazi et al.~\cite{GGKM20}, which has a slightly stronger guarantee than that of~\cite{BunNS19}: it outputs a list of size at most $2^{\poly(d)}$ with the guarantee that $h^*$ belongs to it with probability at least $1/\poly(d)$. Each user runs such an algorithm $\poly(d)$ times on fresh samples drawn from $\MD$, and uses the output hypotheses to build the distribution $\MP_i$.  For pure-DP learners (\Cref{thm:local-dp-learner-generic,thm:pure-dp-learner-generic}), each user simply uses the empirical error on the probabilistic representation of the class to build the distribution $\MP_i$. 

Finally, our SQ algorithm (\Cref{thm:sq-reduction}) deviates slightly from this framework. Instead of computing $\MP_i$ outright (which is usually inefficient since its support is large), we proceed one statistical query at a time. Specifically, each user employs correlated sampling to answer each statistical query; if they manage to answer all queries in the same manner, then the algorithm will output the same hypothesis. The point here is that, since the answer to each statistical query is just a bounded-precision number in $[0, 1]$, building a probability distribution of the possible answers can be done efficiently.

\section{Preliminaries}
\label{sec:prelim}

Let $Y = \{ 0, 1 \}$.  For a set $\Omega$, we use $2^{\Omega}$ to denote the set of all functions from $\Omega$ to $Y$ and use
$\Delta_{\Omega}$ to denote the set of all distributions on $\Omega$.  For distributions $\MP, \MQ$, we use $p \sim \MP$ to denote that
$p$ is drawn from $\MP$ and $d_{\TV}(\MP, \MQ)$ to denote the \emph{total variation distance}
between $\MP$ and $\MQ$.  

Let $X$ be a finite set.%
\footnote{While our results can be extended to the case where $X$ is infinite, it does require non-trivial generalization of notation and tools (e.g., correlated sampling) to that setting.}. 
Let $C$ denote the set of concepts from $X$ to $Y$, and let $\MD$ be any distribution on $X \times Y$ realizable by some $h \in C$.  

We first recall the notion of DP.  Let $\epsilon, \delta \in \mathbb{R}_{\geq 0}$.  Two datasets are \emph{neighboring} if one can be obtained from the other by adding or removing a single user.
\begin{defn}[Differential Privacy (DP)~\cite{DworkMNS06,dwork2006our}]\label{def:dp}
A randomized algorithm $\BA$ taking as input a dataset is \emph{$(\epsilon, \delta)$-differentially private} ($(\epsilon, \delta)$-DP or \emph{approximate-DP}) if for any two \emph{neighboring} datasets $D$ and $D'$, and for any subset $S$ of outputs of $\BA$, it holds that $\Pr[\BA(D) \in S] \le e^{\epsilon} \cdot \Pr[\BA(D') \in S] + \delta$.  If $\delta = 0$, then $\BA$ is \emph{$\epsilon$-differentially private} ($\epsilon$-DP or \emph{pure-DP}).
\end{defn}

A dataset in our setting consists of $n$ users, where user $i$ receives a sequence of $m$ samples $(x^i_1, y^i_1), \dots, (x^i_m, y^i_m)$ drawn i.i.d. from $\MD$. Similar to the standard PAC setting~\cite{Valiant84}, the algorithm $\BA$ takes in the dataset and outputs a hypothesis $f$. We say that it is an \emph{$(\alpha, \beta)$-accurate learner} if $\err{\MD}{f} \leq \alpha$ with probability $1 - \beta$ (where $\alpha, \beta \in (0, 1)$ are parameters); here, $\err{\MD}{f} = \Pr_{(x, y) \sim \MD}[f(x) \neq y]$.  We use $\TIME(\BA)$ to denote the running time of $\BA$.  If $\BA$ is randomized, sometimes we use the notation $\BA(\cdot; r)$ to explicitly call out the (public) randomness $r$ it might use.  

When each user holds exactly a single example (i.e., $m = 1$), we call this the \emph{item-level} setting.  We show results in both the (usual) central and local%
\footnote{
A DP algorithm in the \emph{local} model consists of a randomizer whose input is the samples held by one user and whose output is a sequence of messages, and an analyzer, whose input is the concatenation of the messages from all the randomizers and  whose output is the output of the algorithm.  An algorithm is DP in the local model if for any dataset, the concatenation of the outputs of all the randomizers is DP.
} models of DP.

All missing proofs are in the Supplementary Material.

\subsection{Correlated Sampling}

\begin{defn}[Correlated Sampling]
A \emph{correlated sampling} strategy for a set $\Omega$ with multiplicative error $\kappa$ is an algorithm $\CS: \Delta_{\Omega} \times \MR' \to \Omega$ and a distribution $\MR'$ on random strings such that 
\begin{itemize}
\item (Marginal Correctness) For all $\MP \in \Delta_{\Omega}$ and $\omega \in \Omega$, $\Pr_{r' \sim \MR'}[\CS(\MP; r') = \omega] = \MP(\omega)$.
\item (Error Guarantee) For $\MP, \MQ \in \Delta_\Omega$, $\Pr_{r' \sim \MR'}[\CS(\MP; r') \ne \CS(\MQ; r')] \leq \kappa \cdot d_{\TV}(\MP, \MQ)$.
\end{itemize}
\end{defn}

\begin{theorem}[\cite{Broder97,KleinbergT02,Holenstein07}] \label{thm:correlated-sampling}
For any finite set $\Omega$, there exists a correlated sampling strategy for $\Omega$ with multiplicative error 2.
\end{theorem}

\subsection{Representation Dimension}

The \emph{size} of a hypothesis class $H$ is defined as $\size(H) := \log |H|$, and the size of a distribution $\MH$ of hypothesis classes is defined as $\size(\MH) := \max_{H \in \supp(\MH)} \size(H)$.

\begin{defn}[Probabilistic Representation Dimension~\cite{BeimelNS19}]
A distribution $\MH$ on $2^X$ is said to \emph{$(\alpha, \beta)$-probabilistically represent} a concept class $C$ if for every $f \in C$ and for every distribution $\MD$ on $X$, with probability $1 - \beta$ over $H \sim \MH$, there exists $h \in H$ such that $\Pr_{x \sim \MD}[f(x) \ne h(x)] \leq \alpha$.
The \emph{$(\alpha, \beta)$-probabilistic representation dimension} of a concept class $C$ is defined as
\begin{align*}
\prdim_{\alpha, \beta}(C) := \min_{\MH \text{ that } (\alpha, \beta)\text{-probabilistically represents } C} \size(\MH).
\end{align*}
We use $\prdim(C)$ as a shorthand for $\prdim_{1/4,1/4}(C)$.
\end{defn}

%Feldman and Xiao \cite{FeldmanX15} showed that $\prdim$ is equal to the communication complexity of evaluating a function in a class $C$ on an element in $X$; for our purposes, we will use the above definition.

\begin{lemma}[\cite{BeimelNS19}] \label{lem:prdim-boost}
For every concept class $C$ and $\alpha, \beta > 0$, we have
\begin{align*}
\prdim_{\alpha, \beta}(C) \leq O\left(\log(1/\alpha) \cdot (\prdim(C) + \log\log\log(1/\alpha) + \log\log(1/\beta))\right).
\end{align*}
\end{lemma}
For a concept class $C$, let $\Ldim(C)$ denote its
\emph{Littlestone dimension}~\cite{littlestone_learning_1987}.

\subsection{Tools from DP}
\label{sec:dp-tools}

In the \emph{selection} problem, each user $i$ receives an element $u_i$ from a universe $U$. For each $u \in U$, define $c_u := |\{i \in [n] \mid u_i = u\}|$. The goal is to output $u^*$ such that $c_{u^*} \geq \max_{u \in U} c_u - \alpha$; when the output satisfies this with probability $1 - \beta$, the algorithm is said to be \emph{$(\alpha, \beta)$-accurate}.

% The goal is to compute, for every $u \in U$, $c_u := |\{i \in [n] \mid u_i = u\}|$. We say that a randomized algorithm is $(\alpha, \beta)$-accurate for the histogram problem if it outputs an estimate $\tc$ such that, with probability $1 - \beta$, we have $\max_{u \in U} |c_u - \hat{c}_u| \leq \alpha$.

\begin{lemma}[Approximate-DP Selection~\cite{KorolovaKMN09,BunNS19}] \label{lem:apx-selection}
There is an $(\eps, \delta)$-DP $(O(\log(1/\delta)/\eps), 0)$-accurate algorithm for the selection problem in the central model. Moreover, the algorithm runs in $\poly(n, \log |U|)$ time.
\end{lemma}

%The following pure-DP histogram algorithm in the central model follows from the Laplace mechanism~\cite{DworkMNS06}; an analysis for the guarantee stated below can be found in~\cite[Proposition 2.8]{Vadhan17}.

The following pure-DP histogram algorithm in the central model follows from the exponential mechanism~\cite{McSherryT07}. While a trivial implementation would result in a running time that depends linearly on $|U|$, it is not hard to see that we can first toss a coin to determine whether the output would come from the input set. If so, the sampling can be done in $O(n)$ time; if not, one can randomly output one of the remaining candidates in $U$, which only requires time $O(\log |U|)$. This yields the following.

\begin{lemma}[Pure-DP Selection~\cite{McSherryT07}] \label{lem:pure-selection}
There is an $\eps$-DP $(O(\log(|U|/\beta) / \eps), \beta)$-accurate algorithm for the selection problem in the central model. Moreover, the algorithm runs in $\poly(n, \log |U|)$ time.
\end{lemma}

%The following guarantee can be achieved by combining the RAPPOR algorithm~\cite{ErlingssonPK14} with a standard concentration inequality (e.g. Hoeffding inequality).

The next guarantee follows from the heavy-hitters algorithm of Bassily et al.~\cite{bassily2017practical}:

\begin{lemma}[Pure-DP Histogram in the Local Model~\cite{erlingsson2014rappor}] \label{lem:local-selection}
There is an $\eps$-DP $\left(O\left(\sqrt{n \cdot \log(|U|/\beta)} / \eps\right), \beta\right)$-accurate algorithm for the histogram problem in the local model. Furthermore, the algorithm runs in $\poly(n, \log |U|)$ time.
\end{lemma}

\section{Global Stability and Pseudo-Global Stability}
  
We recall the notion of global stability of Bun et al.~\cite{BunLM20} and generalize it in two ways.
\begin{defn}[Global Stability~\cite{BunLM20}]
A learner $\BA$ is said to be \emph{$m$-sample $\alpha$-accurate $\eta$-globally stable} if there exists a hypothesis $h$ (depending on $\MD$) such that $\Err_{\MD}(h) \leq \alpha$ and \\$\Pr_{(x_1, y_1), \dots, (x_m, y_m) \sim \MD}[\BA((x_1, y_1), \dots, (x_m, y_m)) = h] \geq \eta$.
\end{defn}
We now present the first generalization.  Let $\MR$ be a distribution of random strings.  
\begin{defn}[Pseudo-Global Stability] \label{defn:pseudo-stable}
A learner $\BA$ is said to be \emph{$m$-sample $(\alpha, \beta)$-accurate $(\eta, \nu)$-pseudo-globally stable} if there exists a hypothesis $h_r$ for every $r \in \supp(\MR)$ (depending on $\MD$) such that $\Pr_{r \sim \MR}[\Err_{\MD}(h_r) \leq \alpha] \geq 1 - \beta$ and
\begin{align*}
\Pr_{r \sim \MR}\left[\Pr_{(x_1, y_1), \dots, (x_m, y_m) \sim \MD}[\BA((x_1, y_1), \dots, (x_m, y_m); r) = h_r] \geq \eta\right] \geq \nu.
\end{align*}
\end{defn}
We also generalize global stability in a slightly different manner, in order to capture the guarantees of~\cite{GGKM20}.
\begin{defn}[List Global Stability]
A learner $\BA$ is said to be \emph{$m$-sample $\alpha$-accurate $(L, \eta)$-list globally stable} if $\BA$ outputs a set of at most $L$ hypotheses and there exists a hypothesis $h$ (depending on $\MD$) such that $\Pr_{(x_1, y_1), \dots, (x_m, y_m) \sim \MD}[h \in \BA((x_1, y_1), \dots, (x_m, y_m))] \geq \eta$ and $\Err_{\MD}(h) \leq \alpha$.
\end{defn}

\subsection{Learners with Global Stability}

Bun et al.~\cite{BunLM20} give a globally stable learner in terms of the Littlestone dimension:

\begin{theorem}[\cite{BunLM20}]
Let $\alpha > 0$ and $C$ be any concept class with $\Ldim(C) = d$. Then, there exists a $(2^{O(d)}/\alpha)$-sample $\alpha$-accurate $2^{-O(d)}$-globally stable learner for $C$.
\end{theorem}

Although not explicitly stated in this manner, the improved result of Ghazi et al.~\cite{GGKM20} proceeds by giving a list globally stable learner, where the stability parameter $\eta$ is $\Omega(1/d)$, the list size $L$ is $2^{(d/\alpha)^{O(1)}}$, and the sample complexity is $(d/\alpha)^{O(1)}$.

\begin{theorem}[\cite{GGKM20}] \label{thm:list-stable}
Let $\alpha > 0$ and $C$ be any concept class with $\Ldim(C) = d$. Then, there is a $(d/\alpha)^{O(1)}$-sample $\alpha$-accurate $\left(\exp\left(\left(d/\alpha\right)^{O(1)}\right), \Omega(1/d)\right)$-list globally stable learner for $C$.
\end{theorem}

We will need a slight strengthening of the above result, where there is another parameter $\zeta > 0$ and we want to ensure that every hypothesis in the output list has  error at most $2\alpha$. This is stated below.

\begin{lemma} \label{lem:list-stable-with-restricted-accuracy}
Let $\alpha, \zeta > 0$ and $C$ be any concept class with $\Ldim(C) = d$. Then, there is a $(d \log(1/\zeta)/\alpha)^{O(1)}$-sample $\alpha$-accurate $\left(\exp\left(\left(d/\alpha\right)^{O(1)}\right), \Omega(1/d)\right)$-list globally stable learner for $C$ such that with probability $1 - \zeta$, every hypothesis $h'$ in the output list satisfies $\Err_\MD(h') \leq 2\alpha$.
\end{lemma}

\begin{proof}[Proof Sketch]
This can be done by first running the algorithm in \Cref{thm:list-stable} to get a set $H$ of size at most $L = \exp\left(\left(d/\alpha\right)^{O(1)}\right)$. Then, we draw additional $100 \cdot \log(L/\zeta) / \alpha^2$  samples $S$. Finally, we output $H' = \{h' \in H \mid \Err_S(h') \leq 1.5\alpha\}$. By the Chernoff bound, with probability $1 - \zeta$, every hypothesis $h' \in H$ satisfies $|\Err_\MD(h') - \Err_S(h')| \leq 0.5\alpha$, which yields the desired guarantees.
\end{proof}

\section{Approximate-DP Learner}

In this section, we prove \Cref{thm:apx-dp-learner-generic}.  We first show how to go from list global stability to pseudo-global stability using correlated sampling (\Cref{thm:stability-to-pseudo}).  We then show how to go from pseudo-global stability to an approximate-DP learner using DP selection (\Cref{thm:stability-to-apx-dp}).

\subsection{From List Global Stability to Pseudo-Global Stability}

\begin{theorem} \label{thm:stability-to-pseudo}
Let $\alpha, \beta, \eta \in (0, 0.1), L \in \BN$, and $C$ a concept class. Suppose that there exists a learner $\BA$ that is $m$-sample $\alpha/2$-accurate $(L, \eta)$-list globally stable. Furthermore, with probability $1 - \left(\frac{\beta \cdot \eta^2}{10^6 \cdot \log(L/\eta)}\right)^2$, every hypothesis $h'$ in the output list satisfies $\Err_\MD(h') \leq \alpha$.  Then, there exists a learner $\BA'$ that is $m'$-sample $(\alpha, \beta)$-accurate $(1 - \beta, 1 - \beta)$-pseudo-globally stable, where $m' =  O_{\beta}\left(m \cdot \log^3(L/\eta) / \eta^2\right)$.
\end{theorem}

Before we prove \Cref{thm:stability-to-pseudo}, we note that with \Cref{lem:list-stable-with-restricted-accuracy}, it gives the following corollary.

\begin{corollary} \label{cor:stability-from-littlestone}
Let $\alpha, \beta \in \mathbb{R}_{> 0}$ and $C$ be any concept class with finite $\Ldim(C)$. Then, there exists a learner $\BA$ that is $m$-sample $(\alpha, \beta)$-accurate and $(1 - \beta, 1 - \beta)$-pseudo-globally stable, where $m = O_{\beta}((\Ldim(C) / \alpha)^{O(1)})$.
\end{corollary}

\begin{proof}[Proof of \Cref{thm:stability-to-pseudo}]
Let $\tau = 0.5\eta, \gamma = \frac{10^6 \log(L/(\beta \tau))}{\tau}, k_1 = \frac{10^6 \log(L/(\beta \tau))}{\tau^2}$, and $k_2 = \lceil \frac{10^6 \gamma^2 \cdot \log(L/(\tau \beta))}{\beta^4} \rceil$. Let $\CS$ be a correlated sampling strategy for $2^X$ and let $\MR'$ be the (public) randomness it uses, as in \Cref{thm:correlated-sampling}.  \Cref{alg:approxstable} presents our learner $\BA'$.

\begin{algorithm}[ht]
\caption{Pseudo-Globally Stable Learner $\BA'$.
\label{alg:approxstable}}
%\small
\begin{algorithmic}
\FOR{$i = 1, \dots, k_1$}
\STATE Draw $S_i \sim \MD^m$, run $\BA$ on $S_i$ to get a set $H_i$
\ENDFOR
\STATE Let $H$ be the set of all $f \in 2^X$ that appears in at least $\tau \cdot k_1$ of the sets $H_1, \dots, H_{k_1}$
\FOR{$j = 1, \dots, k_2$}
\STATE
Draw $T_j \sim \MD^m$, run $\BA$ on $T_j$ to get a set $G_j$
\ENDFOR
\FOR{$h \in H$}
\STATE $\mbox{Let } \hat{Q}_{H, G_1, \dots, G_{k_2}}(h) = \frac{|\{j \in [k_2] \mid h \in G_j\}|}{k_2}$
\ENDFOR
\STATE Let $\hat{\MP}_{H, G_1, \dots, G_{k_2}}$ be the probability distribution on $2^X$ defined by
\begin{align*}
\hat{\MP}_{H, G_1, \dots, G_{k_2}}(h) =
\begin{cases}
\frac{\exp(\gamma \cdot \hat{Q}_{G_1, \dots, G_{k_2}}(h))}{\sum_{h' \in H} \exp(\gamma \cdot \hat{Q}_{G_1, \dots, G_{k_2}}(h'))} & \text{ if } h \in H, \\
0 & \text{ otherwise.}
\end{cases}
\end{align*}
\STATE Output $\CS(\hat{\MP}_{H, G_1, \dots, G_j}; r')$, where $r' \sim \MR'$
\end{algorithmic}
\end{algorithm}

Notice that the number of samples used in $\BA'$ is $m \cdot (k_1 + k_2) = m \cdot O_{\beta}(\log^3(L/\eta) / \tau^2)$ as claimed.

\paragraph{(Accuracy Analysis)}
Since we assume that the output of $\BA$ consists only of hypotheses with distributional error at most $\alpha$ with probability $1 - \beta / k_1$, a union bound implies that this holds for all hypotheses in $H$ with probability $1 - \beta$. This yields the desired $(\alpha, \beta)$-accuracy of the algorithm.

\paragraph{(Pseudo-Global Stability Analysis)}
For this, we need a few additional notation. First, for every $h \in 2^X$, we let $Q(h)$ denote $\Pr_{S \sim \MD^m}[h \in \BA(S)]$. Moreover, let $H_{\geq 1.1\tau} = \{h \in 2^X \mid Q(h) \geq 1.1\tau\}$ and similarly $H_{\geq 0.9\tau} = \{h \in 2^X \mid Q(h) \geq 0.9\tau\}$. A crucial property we will use is that $H$ is w.h.p. sandwiched between $H_{\geq 1.1\tau}$ and $H_{\geq 0.9\tau}$, as stated below.

\begin{lemma} \label{lem:candidate-sandwich}
Let $\mathscr{E}$ denote the event that $H_{\geq 1.1\tau} \subseteq H \subseteq H_{\geq 0.9\tau}$. Then,  
\begin{align*}
\Pr[\mathscr{E}] \geq 1 - \beta^2 / 30,
\end{align*}
where the probability is over the randomness of $S_1, \dots, S_{k_1}$ and that of $\BA$ on these datasets.
\end{lemma}

\begin{proof}[Proof of \Cref{lem:candidate-sandwich}]
We will separately argue that $\Pr[H_{\geq 1.1\tau} \nsubseteq H] \leq \beta^2 / 60$ and $\Pr[H \nsubseteq H_{\geq 0.9\tau}] \leq \beta^2 / 60$.  A union bound then yields the claimed statement.

To prove the first bound, observe that since $\BA$ outputs a set of size at most $L$, $|H_{\geq 1.1\tau}| \leq L/(1.1\tau) < L/\tau$. Consider each $f \in H_{\geq 1.1\tau}$; notice that $\bone[f \in H_i]$ is simply an i.i.d. Bernoulli random variable with success probability  $Q(f) \geq 1.1\tau$. Hence, by the Hoeffding inequality, we have
\begin{align*}
\Pr[f \notin H] \leq \exp\left(-0.02 \tau^2 k_1\right) < 0.001 \beta^2\tau / L,
\end{align*}
where the last inequality follows from our choice of $\tau, k_1$. Taking a union bound over all $f \in H_{\geq 1.1\tau}$ concludes our proof for the first inequality.

For the second inequality, consider the set $H_{< 0.9\tau} := 2^X \setminus H_{\geq 0.9\tau}$. Since each element $f \in H_{< 0.9\tau}$ satisfies $Q(f) < 0.9\tau$, we may partition\footnote{A simple way is to start with a singleton partition and then merge any two parts whose total $Q(\cdot)$ is less than $0.9\tau$; in the end, we will left with a partition where all but at most one part has weight at least $0.45\tau$.} $H_{< 0.9\tau}$ into $H^1_{<0.9\tau} \cup \cdots \cup H^q_{0.9\tau}$ such that $\sum_{f \in H^j_{<0.9\tau}} Q(f) < 0.9\tau$ for all $j \in [q]$ and $q \leq L/(0.45\tau) + 1 < 4L/\tau$. Fix $j \in [q]$; notice that
\begin{align*}
\Pr[H \cap H^j_{< 0.9\tau} \ne \emptyset] \leq \Pr[|\{i \in [k_1] \mid H_i \cap H^j_{< 0.9\tau} \ne \emptyset\}| \geq \tau k_1].
\end{align*}
Now, each $\bone[H_i \cap H^j_{< 0.9\tau} \ne \emptyset]$ is an i.i.d. Bernoulli random variable with success probability at most $\sum_{f \in H^j_{< 0.9\tau}} Q(f) < 0.9\tau$. Thus, we can apply the Hoeffding inequality to conclude that
\begin{align*}
\Pr[H \cap H^j_{< 0.9\tau} \ne \emptyset] \leq \exp(-0.02\tau^2 k_1) < 0.001\beta^2 \tau / L.
\end{align*}
Taking a union bound over all $j \in [q]$, we have $\Pr[H \cap H_{< 0.9\tau} \ne \emptyset] < 0.01\beta^2$. This completes our proof of the second inequality.
\end{proof}

Next, let $\MP$ be the probability distribution on $2^X$ defined by
\begin{align*}
\MP(f) =
\begin{cases}
\frac{\exp(\gamma \cdot Q(f))}{\sum_{f' \in H_{\geq 0.9\tau}} \exp(\gamma \cdot Q(f'))} & \text{ if } f \in H_{\geq 0.9\tau} \\
0 & \text{ otherwise.}
\end{cases}.
\end{align*}
Furthermore, let $\MP_H$ be the probability distribution on $2^X$ defined by
\begin{align*}
\MP_H(f) =
\begin{cases}
\frac{\exp(\gamma \cdot Q(f))}{\sum_{f' \in H} \exp(\gamma \cdot Q(f'))} & \text{ if } f \in H, \\
0 & \text{ otherwise.}
\end{cases}
\end{align*}
Once again, notice that $\MP$ is independent of the run of the algorithm (i.e., it only depends on $\BA$), whereas $\MP_H$ can vary on different runs, depending on $H$. Our first component of the proof is to argue that $\MP$ and $\MP_H$ are often close:
\begin{lemma} \label{lemma:restriction-prob-close-to-true}
When $\mathscr{E}$ holds, we have $d_{\TV}(\MP, \MP_H) \leq \beta^2 / 30$.
\end{lemma}
\begin{proof}[Proof of \Cref{lemma:restriction-prob-close-to-true}]
Recall that from the assumption of list global stability, there exists $h$ such that $Q(h) \geq \eta = 2\tau$. When $\mathscr{E}$ holds, $H$ is a subset of $H_{\geq 0.9\tau}$, meaning that $\MP_H$ is the conditional probability of $\MP$ on $H$. Thus, we have
\begin{align*}
d_{\TV}(\MP, \MP_H) 
& \leq \MP(H \setminus H_{\geq 0.9\tau}) 
= \frac{\sum_{f \in H_{\geq 0.9\tau} \setminus H} \exp(\gamma \cdot Q(f))}{\sum_{f' \in H_{\geq 0.9\tau}} \exp(\gamma \cdot Q(f'))} 
\overset{(a)}{\leq} \frac{\sum_{f \in H_{\geq 0.9\tau} \setminus H_{\geq 1.1\tau}} \exp(\gamma \cdot Q(f))}{\sum_{f' \in H_{\geq 0.9\tau}} \exp(\gamma \cdot Q(f'))} \\
& \leq \frac{\sum_{f \in H_{\geq 0.9\tau} \setminus H_{\geq 1.1\tau}} \exp(\gamma \cdot 1.1\tau)}{\sum_{f' \in H_{\geq 0.9\tau}} \exp(\gamma \cdot Q(f'))} 
\leq \frac{|H_{\geq 0.9\tau}| \cdot \exp(\gamma \cdot 1.1\tau)}{\exp(\gamma \cdot Q(h))} \\
& \overset{(b)}{\leq} |H_{\geq 0.9\tau}| \cdot \exp(-\gamma \cdot 0.9\tau) 
\leq L / (0.9\tau) \cdot \exp(-\gamma \cdot 0.9 \tau) 
\overset{(c)}{\leq} \beta^2 / 30,
\end{align*}
where inequality (a) follows from $H_{\geq 1.1\tau} \subseteq H$ (which in turns holds because of $\mathscr{E}$), inequality (b) follows since $Q(h) \geq 2\tau$, and inequality (c) follows from our choice of $\gamma$.
\end{proof}

Next, we show that $\MP_H$ is often close to its ``empirical'' version $\hat{\MP}_{H, G_1, \dots, G_{k_2}}$.
\begin{lemma} \label{lemma:restriction-prob-close-to-empirical}
$\E[d_{\TV}(\MP_H, \hat{\MP}_{H, G_1, \dots, G_{k_2}})] \leq \beta^2 / 30$ where the probability is over the randomness of $T_1, \dots, T_{k_2}$ and that of $\BA$'s executions on these datasets.
\end{lemma}

\begin{proof}[Proof of \Cref{lemma:restriction-prob-close-to-empirical}]
From how $H$ is selected and from the assumption that the output of $\BA$ has size at most $L$, we have $|H| \leq L / \tau$. Now, fix $f \in H$.  Note that $\hat{Q}_{H, G_1, \dots, G_{k_2}}(f)$ is simply an average of $k_2$ i.i.d. Bernoulli random variables with success probability $Q(f)$.  Using the Hoeffding inequality,
\begin{align} \label{eq:freq-diff}
\Pr\left[|\hat{Q}_{H, G_1, \dots, G_{k_2}}(f) - Q(f)| > 100\sqrt{\frac{\log(L/(\tau \beta))}{k_2}}\right] \leq \frac{\beta^2}{60 \cdot (L / \tau)}.
\end{align}
By a union bound over all $f \in H$, we can conclude that with probability $1-\beta^2/60$ we have $|\hat{Q}_{H, G_1, \dots, G_{k_2}}(f) - Q(f)| \leq 100\sqrt{\frac{\log(L/(\tau \beta))}{k_2}}$ for all $f \in H$. When this holds, we have
\begin{align*}
& d_{\TV}(\MP_H, \hat{\MP}_{H, G_1, \dots, G_{k_2}}) = \sum_{h \in h} \MP_H(h) \cdot \max\left\{0, \left(\frac{\hat{\MP}_{H, G_1, \dots, G_{k_2}}(h)}{\MP_H(h)} - 1\right)\right\} \\
& \overset{\eqref{eq:freq-diff}}{\leq} \sum_{h \in h} \MP_H(h) \cdot \left(\exp\left({\gamma \cdot 100\sqrt{\frac{\log(L/(\tau \beta))}{k_2}}}\right) - 1\right) 
\leq \exp\left({\gamma \cdot 100\sqrt{\frac{\log(L/(\tau \beta))}{k_2}}}\right) - 1 \\
& \leq 200\gamma \sqrt{\frac{\log(L/(\tau \beta))}{k_2}} \leq \beta^2 / 60,
\text{ from our choice of } k_2.
\qedhere
\end{align*}
\end{proof}

Combining \Cref{lem:candidate-sandwich,lemma:restriction-prob-close-to-true,lemma:restriction-prob-close-to-empirical}, we can conclude that $\E[d_{\TV}(\MP, \hat{\MP}_{H, G_1, \dots, G_{k_2}})] \leq \beta^2/10$ where the expectation is over all the randomness involved in the algorithm except $r'$. Now, let $h_{r'} = \CS(P; r')$. From the error guarantee of correlated sampling in \Cref{thm:correlated-sampling}, we can conclude that $$\Pr_{r' \sim \MR', S_1, \dots, S_{k_1}, T_1, \dots, T_{k_2}}[\BA'(S_1, \dots, S_{k_1}, T_1, \dots, T_{k_2}; r') \ne h_{r'}] \leq \beta^2/5.$$ 
Thus, applying Markov inequality, we get the pseudo-global stability of $\BA$
\[
\Pr_{r' \sim \MR'}\left[\Pr_{S_1, \dots, S_{k_1}, T_1, \dots, T_{k_2}}[\BA'(S_1, \dots, S_{k_1}, T_1, \dots, T_{k_2}; r') = h_{r'}] \geq 1 - \beta\right] \geq 1 - \beta.
\qquad
\qedhere
\]
\end{proof}

\subsection{From Pseudo-Global Stability to Approximate-DP Learner in the Central Model}

In this section we prove \Cref{thm:stability-to-apx-dp} that allows us to convert any pseudo-globally stable learner to an approximate-DP learner. Combining  \Cref{thm:stability-to-apx-dp} and \Cref{cor:stability-from-littlestone} yields Theorem~\ref{thm:apx-dp-learner-generic}.

\begin{theorem} \label{thm:stability-to-apx-dp}
Let $\alpha, \beta \in (0, 0.1)$ and $C$ a concept class. Suppose that there exists a learner $\BA$ that is $m$-sample $(\alpha, \beta/3)$-accurate $(0.9, 1-\beta/3)$-pseudo-globally stable.  Then, for any $\eps \in (0, 1)$, there is an $(\eps, \delta)$-DP $(\alpha, \beta)$-accurate learner for $C$ in the central model that requires $n = O\left(\log(1/(\delta \beta)) / \eps\right)$ users where each user has $m$ samples.  The learner runs in time $\poly(\TIME(\BA), n, d)$.
\end{theorem}

\begin{proof}
Let $n = K \cdot \log\left(1/(\delta \beta)\right) / \eps$ where $K$ is a sufficiently large constant.
Let $\MR$ denote the public randomness shared between the users.  \Cref{alg:approxDP} shows our approximate-DP learner.
\begin{algorithm}[H]
\caption{Approximate-DP Learner in the Central Model. \label{alg:approxDP}}
%\small
\begin{algorithmic}
\STATE User $i$ randomly draws $m$ samples $S_i \sim \MD^m$ and runs $\BA(S_i; r)$ to obtain a hypothesis $h_i$
\STATE Run the $(\eps, \delta)$-DP selection algorithm from~\Cref{lem:apx-selection} where $U = 2^X$ and user $i$'s item is $h_i$ %Let $(\tc_h)_{h \in H_r}$ denote the estimate histogram output by the algorithm.
\STATE Return the output $h^*$ from the previous step
\end{algorithmic}
\end{algorithm}
It is clear that \Cref{alg:approxDP} is $(\eps, \delta)$-DP. We will now analyze its accuracy. First, from the definition of pseudo-global stability and a union bound, with probability $1 - 2\beta/3$ over $r \sim \MR$, there exists $h_r$ such that $\Err_{\MD}(h_r) \leq \alpha$, and $\Pr_{S \sim \MD^m}[\BA(S; r) = h_r] \geq 0.9$. Conditioned on this, we can use the Hoeffding inequality to conclude that, with probability $1 - \beta/6$, we have\footnote{Recall that $c_h = |\{i \in [n] \mid h_i = h\}|$ is as defined in \Cref{sec:dp-tools}.}
\begin{align} \label{eq:correct-large-apx}
c_{h_r} \geq 0.8 n.
\end{align}
Conditioned on~\eqref{eq:correct-large-apx}, \Cref{lem:apx-selection} guarantees that $c_{h^*} \geq \max_{h \in 2^X}  c_h - 0.1 n$ with probability $1 - \beta/6$ (for sufficiently large $K$); when this is the case, \Cref{alg:approxDP} outputs $h^* = h_r$. Hence, applying a union bound (over this and~\eqref{eq:correct-large-apx}), we can conclude that the algorithm is $(\alpha, \beta)$-accurate as desired.
%we can conclude that the algorithm outputs a hypothesis with error at most $\alpha$ with respect to $\MD$ as desired.
\end{proof}

\section{Conclusions and Future Directions}
\label{sec:conc}

In this work, we study the question of learning with user-level DP when each user may have many examples. We prove tight upper and lower bounds on the \emph{number of users} required to learn each concept class, provided that each user has sufficiently many i.i.d. samples. An immediate open question here is whether one can also derive a tight bound on the \emph{number of samples per users} required; note that this bound will depend on the number of users. For approximate-DP learning, this problem might be hard because the big gap in the item-level learning setting between $\poly(\Ldim)$~\cite{GGKM20} and $\Omega(\log^* \Ldim)$~\cite{AlonLMM19}) is still open.  The pure-DP case might be easier since tight bounds are known both in the central~\cite{BeimelNS19} and the local models (up to a polynomial factor)~\cite{kasiviswanathan2011can}.

Another interesting direction is to derive additional efficient user-level DP learners whose sample complexities are better than item-level DP learners. We give two ``weak'' examples of this in \Cref{sec:sq-efficient} for the case of the non-interactive local model. It would be good to give such an example for general algorithms in the local model as well. On this front, PARITY seems to be a good candidate; as stated earlier, it has SQ dimension $2^{\Omega(d)}$~\cite{BlumFJKMR94} meaning that it requires $2^{\Omega(d)}$ samples in the (interactive) item-level local model~\cite{kasiviswanathan2011can}. Can we come up with an efficient user-level DP algorithm in the local model that requires $\poly(d)$ samples in total?

Furthermore, our algorithms make extensive use of shared randomness---in the form of correlated sampling.  Is this necessary?  In particular,
\begin{itemize}
\item Is there a local user-level DP algorithm for learning any class $C$ using $\poly(\prdim(C))$ users each having $\poly(\prdim(C))$ samples, without using public randomness? In other words, can the use of public randomness be removed from \Cref{thm:local-dp-learner-generic}?
\item Is there an $\eta$-globally stable learner for any class $C$ with finite Littlestone dimension where $\eta > 0$ is some absolute constant? In other words, can the use of public randomness be removed from \Cref{cor:stability-from-littlestone}?
\end{itemize}
While it is not hard to show that the second question has a negative answer if we require $\eta > 1/2$, we are not aware of a proof that $\eta$ must go to zero for some family of concept classes.

Lastly, it would also be interesting to see whether techniques employed in our paper may be useful beyond the PAC setting. For example, Golowich~\cite{Golowich21} gives DP regression algorithms based on stability notions similar to~\cite{BunLM20,GGKM20} and it is plausible that our approach gives user-level regression algorithm in a setting similar to~\cite{Golowich21}.

\subsection*{Acknowledgement}
We thank Jessica Sorrell for pointing us to~\cite{reproducibility} and explaining the main results of that paper.

\bibliographystyle{alpha}
\bibliography{refs}

%\ifneurips
%\input{checklist}
%\fi

\newpage

\ifneurips
\setcounter{page}{1}
\fi

\appendix

\ifneurips
\section*{Supplementary Material}
\fi

\section{Pure-DP Learners}

In this section we prove \Cref{thm:local-dp-learner-generic} and \Cref{thm:pure-dp-learner-generic}.  First we show how to go from probabilistic representation to pseudo-global stability (\Cref{thm:stability-from-representation}) using correlated sampling.  Next we show how to go from pseudo-global stability to a pure-DP learner using the histogram algorithm in the local (\Cref{thm:stability-to-local-dp}) model and the central (\Cref{thm:stability-to-pure-dp}) model.

\subsection{Pseudo-Globally Stable Learner from Probabilistic Representation}

\begin{theorem} \label{thm:stability-from-representation}
Let $\alpha, \beta \in \mathbb{R}_{> 0}$ and $C$ be any concept class with finite $\prdim(C)$. Then, there exists a learner $\BA$ that is $m$-sample $(\alpha, \beta)$-accurate and $(1 - \beta, 1 - \beta)$-pseudo-globally stable, where $m = O_{\beta}\left(\prdim_{\alpha/2,\beta/2}(C)^3/\alpha^2\right)$. Furthermore, the public randomness specifies $H \in \supp(\MH)$ such that the output of $\BA$ belongs to $H$.
\end{theorem}

Before we prove \Cref{thm:stability-from-representation}, we note that this along with \Cref{lem:prdim-boost} gives the following corollary.

\begin{corollary} \label{cor:stability-from-representation}
Let $\alpha, \beta \in \mathbb{R}_{> 0}$ and $C$ be any concept class with finite $\prdim(C)$. Then, there exists a learner $\BA$ that is $m$-sample $(\alpha, \beta)$-accurate and $(1 - \beta, 1 - \beta)$-pseudo-globally stable, where $m = \tilde{O}_{\beta}\left(\prdim(C)^3/\alpha^2\right)$. Furthermore, the public randomness specifies $H \in \supp(\MH)$ such that the output of $\BA$ belongs to $H$. 
\end{corollary}

\begin{proof}[Proof of \Cref{thm:stability-from-representation}]
For brevity, let $d := \prdim_{\alpha/2, \beta/2}(C)$ and $\gamma = \frac{2(d + \log(1/\beta) + 10)}{\alpha}$. 
Let $\CS$ be a correlated sampling strategy for $2^X$ and let $\MR'$ be the (public) randomness it uses, as in \Cref{thm:correlated-sampling} and let $\MH$ be an $(\alpha/2, \beta/2)$-probabilistic representation of $C$ such that $\size(\MH) = d$.

The public randomness used in our learner $\BA$ is split into two parts: $r' \sim \MR'$ and $H \sim \MH$. 
\Cref{alg:purestable} presents the pseudo-globally stable learner $\BA$.  

\begin{algorithm}[H]
\caption{Pseudo-Globally Stable Learner $\BA$.
\label{alg:purestable}}
\begin{algorithmic}
\STATE Draw $m := \frac{10^6 \gamma^2 d \log(1/\beta)}{\beta^4}$ samples from $\MD$; let $S$ denote the multiset of these samples
\STATE Let $\hat{\MP}_{H, S}$ denote the probability distribution on $\Delta_{2^X}$ where
\begin{align*}
\hat{\MP}_{H, S}(h) =
\begin{cases}
\frac{\exp(-\gamma \cdot \Err_S(h))}{\sum_{h' \in H} \exp(-\gamma \cdot \Err_S(h'))} & h \in H, \\
0 & \text{otherwise.}
\end{cases}
\end{align*}
\STATE Output $\CS(\hat{\MP}_{H, S}; r')$
\end{algorithmic}
\end{algorithm}

We remark that the last property in the theorem statement holds simply because $H$ is part of the public randomness and $\BA$ always outputs a hypothesis that belongs to $H$.

To prove the pseudo-global stability and accuracy of the learner, we introduce additional notation: Let $\MP_H$ denote the probability distribution on $\Delta_{2^X}$ where
\begin{align*}
\MP_H(h) =
\begin{cases}
\frac{\exp(-\gamma \cdot \Err_\MD(h))}{\sum_{h' \in H} \exp(-\gamma \cdot \Err_\MD(h'))} & h \in H, \\
0 & \text{otherwise.}
\end{cases}
\end{align*}
In other words, $\MP_H$ is the distributional version of (the empirical) $\hat{\MP}_{H, S}$.

\paragraph{(Pseudo-Global Stability Analysis)} 
To show $(1 - \beta, 1 - \beta)$-pseudo-global stability, we start by bounding the total variation distance between $\MP_H$ and $\hat{\MP}_{H, S}$.
\begin{lemma} \label{lemma:analysis-after-good-representation}
For a fixed $H \in \supp(\MH)$, we have
\begin{align} \label{eq:rep-tv-bound}
\Pr_{S \sim \MD^m}[d_{\TV}(\MP_H, \hat{\MP}_{H, S}) \leq \beta / 10] \geq 1 - \beta^2/10.
\end{align}
\end{lemma}

\begin{proof}[Proof of \Cref{lemma:analysis-after-good-representation}]
By the Hoeffding inequality, for each $h \in \MH$, we have $\Pr_{S \sim \MD^m}[|\Err_{\MD}(h) - \Err_S(h)| \leq 10 \sqrt{d \log(1/\beta)/m}] \leq \frac{\beta^2}{10 \cdot 2^d}$. By a union bound, with probability $1 - \beta^2/10$ we have $|\Err_{\MD}(h) - \Err_S(h)| \leq 10 \sqrt{d\log(1/\beta)/m}$ for all $h \in H$. When this holds, we have
\begin{align*}
d_{\TV}(\MP_H, \hat{\MP}_{H, S}) &= \sum_{h \in h} \MP_H(h) \cdot \max\left\{0, \left(\frac{\hat{\MP}_{H, S}(h)}{\MP_H(h)} - 1\right)\right\} \\
&\leq \sum_{h \in h} \MP_H(h) \cdot \left(\exp\left({\gamma \cdot 20\sqrt{d \log(1/\beta) / m}}\right) - 1\right) \\
&\leq \exp\left({\gamma \cdot 20\sqrt{d \log(1/\beta) / m}}\right) - 1 \\
&\leq 40\gamma\sqrt{d \log(1/\beta) / m} \\
(\text{From our choice of } m) &\leq \beta^2 / 10.
\qedhere
\end{align*}
\end{proof}

Let $h_{(r', H)} = \CS(\MP_H; r')$. Recall from the guarantee of the correlated sampling algorithm that
\begin{align*}
\Pr_{r' \sim \MR'}[\CS(\hat{\MP}_{H, S}; r') \ne h_{(r', H)}] \leq 2 \cdot d_{\TV}(\MP_H, \hat{\MP}_{H, S}).
\end{align*}
Combining the above inequality with~\eqref{eq:rep-tv-bound}, for any fix $H \in \supp(\MH)$, we have
\begin{align*}
\Pr_{r' \sim \MR'}\Pr_{S \sim \MD^m}[\BA(S) \ne h_{(r', H)}] \leq \E_{d_{\TV}}[2 \cdot d_{\TV}(\MP_H, \hat{\MP}_{H, S})] \leq 2 \beta^2 / 5 < \beta^2.
\end{align*}
From this, we can conclude that
\begin{align*}
\Pr_{r' \sim \MR'} \left[\Pr_{S \sim \MD^m}[\BA(S) = h_{(r', H)}] \geq 1 - \beta\right] \geq 1 - \beta,
\end{align*}
for all $H \in \supp(\MH)$. 

\paragraph{(Accuracy Analysis)}
First, recall from the definition of $(\alpha/2, \beta/2)$-probabilistic representation of $C$, with probability $1 - \beta/2$ over $H \sim \MH$, we have
\begin{align} \label{eq:good-representation}
\exists h^* \in H, \Err_{\MD}(h^*) \leq \alpha/2.
\end{align}
When this holds, we can analyze as in~\cite{HardtT10} for the accuracy of the exponential mechanism. More formally, from the marginal correctness of the correlated sampling algorithm, we have
\begin{align*}
\Pr_{r' \sim \MR'}[\Err_{\MD}(h_{r', H}) > \alpha] &= \Pr_{h \sim P_H}[\Err_{\MD}(h) > \alpha] \\
&= \frac{\sum_{h \in H \atop \Err_{\MD}(h) > \alpha} \exp(-\gamma \cdot \Err_\MD(h))}{\sum_{h' \in H} \exp(-\gamma \cdot \Err_\MD(h'))} \\
&\leq \frac{|\MH| \cdot \exp(-\gamma \cdot \alpha)}{\exp(-\gamma \cdot \Err_{\MD}(h^*))} \\
(\text{From~\eqref{eq:good-representation}}) &\leq 2^d \cdot \exp(-\gamma \cdot \alpha / 2) \\
(\text{From our choice of } \gamma) &\leq \beta / 2.
\end{align*}
Combining the above inequality with the fact that~\eqref{eq:good-representation} holds with probability at least $1 - \beta/2$, 
\begin{align*}
\Pr_{r' \sim \MR', H \sim \MH}[\Err_{\MD}(h_{r', H}) > \alpha] \geq 1 - \beta,
\end{align*}
which concludes our proof.
\end{proof}

\subsection{From Pseudo-Globally Stable Learner to Pure-DP Learner in the Local Model}

In this section we prove \Cref{thm:stability-to-local-dp}, which allows us to convert any pseudo-globally stable learner to a pure-DP learner in the local model. Combining this and \Cref{cor:stability-from-representation} yields Theorem~\ref{thm:local-dp-learner-generic}.

\begin{theorem} \label{thm:stability-to-local-dp}
Let $\alpha, \beta \in (0, 0.1)$ and $C$ a concept class. Suppose that there exists a learner $\BA$ that is $m$-sample $(\alpha, \beta/3)$-accurate and $(0.9, 1-\beta/3)$-pseudo-globally stable. Furthermore, suppose that each public randomness $r$ of $\BA$ specifies a hypothesis class $H_r$ of size at most $d$ such that $\BA$ outputs (on that public randomness) always belong to $H_r$.
Then, for any $\eps \in (0, 1)$, there exists an $\eps$-DP $(\alpha, \beta)$-accurate learner in the (public randomness) local model for $C$ that requires $n = O\left(\frac{d + \log(1/\beta)}{\eps^2}\right)$ users where each user has $m$ samples. Moreover, the running time of the learner is $\poly(\TIME(\BA), n, d)$.
\end{theorem}

\begin{proof}
Let $n = K \left(d + \log(1/\beta)\right) / \eps^2$ where $K$ is a sufficiently large constant.  Let $r$ denote the public randomness shared between the users.  \Cref{alg:purelearner} presents our pure-DP learner.
\begin{algorithm}[H]
\caption{Pure-DP Learner in the Local Model.  
\label{alg:purelearner}}
\begin{algorithmic}
\STATE User $i$ draws $m$ samples $S_i \sim \MD^m$
\STATE User $i$ runs $\BA(S_i; r)$ to obtain hypothesis $h_i \in H_r$
\STATE Run the $\eps$-DP selection algorithm in the local model from~\Cref{lem:local-selection}, where $U = H_r$ and user $i$'s item is $h_i$. %Let $(\tc_h)_{h \in H_r}$ denote the estimate histogram output by the algorithm.
\STATE Return the output $h^*$ from the previous step
%\item User $i$ then uses the RAPPOR algorithm~\cite{ErlingssonPK14} to obtain a vector $\ts^i \in \{0, 1\}^{H_r}$. More specifically, first let $s^i \in \{0, 1\}^{H_r}$  be the one-hot encoding of $h$. Then, let
%\begin{align*}
%\ts^i_h =
%\begin{cases}
%s^i_h & \text{ with probability } \frac{e^{\eps}}{1 + e^{\eps}}, \\
%1 - s^i_h & \text{ with probability } \frac{1}{1 + e^{\eps}},
%\end{cases}
%\end{align*}
%for all $h \in H_r$.
%\item Finally, the analyzer outputs $\argmax_{h \in H} \sum_{i \in [n]} \ts^i_h$.
\end{algorithmic}
\end{algorithm}
It is obvious to see that the algorithm is $\eps$-DP in the local model. We will now analyze its accuracy. First, from definition of pseudo-global stability and a union bound, with probability $1 - 2\beta/3$ over $r \sim \MR$ there exists $h_r$ such that $\Err_{\MD}(h_r) \leq \alpha$, and $\Pr_{S \sim \MD^m}[\BA(S; r) = h_r] \geq 0.9$. Conditioned on this, we can use the Hoeffding inequality to conclude that, with probability $1 - \beta/6$, 
\begin{align} \label{eq:correct-large}
c_{h_r} \geq 0.8.
\end{align}
Conditioned on~\eqref{eq:correct-large}, we can use \Cref{lem:local-selection} to guarantee that $c_{h^*} \geq \max_{h \in 2^X}  c_h - 0.1 n$ with probability at least $1 - \beta/6$ (when $K$ is sufficiently large); when this is the case, the algorithm outputs $h^* = h_r$. Hence, applying a union bound (over this and~\eqref{eq:correct-large}), we can conclude that the algorithm outputs a hypothesis with error at most $\alpha$ with respect to $\MD$ as desired.
%
%we may again apply Hoeffding inequality to conclude that, with probability $1 - \frac{\beta}{6 \cdot 2^d}$ we have $\sum_{i \in [n]} \ts^i_h \geq n\left(\frac{1}{1 + e^{\eps}} + \frac{0.7(e^{\eps} - 1)}{e^{\eps} + 1}\right)$. Similarly, conditioned on~\eqref{eq:correct-large}, for every $h \ne h_r$ in $H_r$, we have $\sum_{i \in [n]} s^i_{h_r} \leq 0.2$; thus, Hoeffding inequality implies that with probability $1 - \frac{\beta}{6 \cdot 2^d}$. Applying the union bound, we have that $h_r = \argmax_{h \in H} \sum_{i \in [n]} \ts^i_h$ with probability $1 - \beta/6$. Hence, applying yet another union bound (over this and~\eqref{eq:correct-large}), we can conclude that the algorithm outputs a hypothesis with error at most $\alpha$ with respect to $\MD$ as desired.
\end{proof}

\subsection{From Pseudo-Globally Stable Learner to Pure-DP Learner in the Central Model}

We next prove the following result, which is similar to \Cref{thm:stability-to-local-dp} except that we now work in the central model and achieve a slightly better bound, i.e., the dependency of the number of users on $\eps$ is $1/\eps$ instead of $1/\eps^2$. Plugging \Cref{thm:stability-to-local-dp} into \Cref{cor:stability-from-representation}, we get \Cref{thm:pure-dp-learner-generic}.

\begin{theorem} \label{thm:stability-to-pure-dp}
Let $\alpha, \beta \in (0, 0.1)$ and $C$ a concept class. Suppose that there exists a learner $\BA$ that is $m$-sample $(\alpha, \beta/3)$-accurate and $(0.9, 1-\beta/3)$-pseudo-globally stable. Furthermore, suppose that each public randomness $r$ of $\BA$ specifies a hypothesis class $H_r$ of size at most $d$ such that $\BA$ outputs (on that public randomness) always belong to $H_r$.
Then, for any $\eps \in (0, 1)$, there exists an $\eps$-DP $(\alpha, \beta)$-accurate learner in the central model for $C$ that requires $n = O\left(\frac{d + \log(1/\beta)}{\eps}\right)$ users where each user has $m$ samples. Moreover, the running time of the learner is $\poly(\TIME(\BA), n, d)$.
\end{theorem}

The proof of \Cref{thm:stability-to-pure-dp} is essentially the same as that of \Cref{thm:stability-to-local-dp}, except that we are using the histogram algorithm in the central model (\Cref{lem:pure-selection}) instead of in the local model (\Cref{lem:local-selection}).

\section{Efficient Reduction for SQ Algorithms}
\label{sec:sq-efficient}

This section is devoted to the proof of \Cref{thm:sq-reduction}.
To understand the intuition behind the proof, observe that our algorithms from the previous section are inefficient because they have to estimate a certain probability distribution over the set of possible output hypotheses (on which a correlating sampling is then applied); this set can be large, resulting in the inefficiency.  To overcome this, we observe that in the SQ model our job is now to simply produce a single number---the output of the oracle---which is then returned back to $\BA$.  Since it is a single number (and can have error as large as $\tau$), we can quite easily estimate its value and use correlated sampling to round it to some nearby number. By doing this for every oracle call from $\BA$, we can obtain a pseudo-globally stable algorithm, which can then be turned into user-level DP algorithms using  \Cref{thm:stability-to-apx-dp,thm:stability-to-pure-dp,thm:stability-to-local-dp}.

\subsection{From SQ Algorithms to Pseudo-Globally Stable Learners}

We recall the definition of statistical queries~\cite{Kearns98}.

\begin{defn}[Statistical Query Oracle~\cite{Kearns98}]
For a given distribution $\MD$ and accuracy parameter $\tau > 0$, a \emph{statistical query (SQ)}  $\stat_{\MD}(\tau)$ is an oracle that, when given a function $\phi: \supp(\MD) \to [-1, 1]$, outputs some number $o$ such that $|o - \E_{z \sim \MD}[\phi(z)]| \leq \tau$.
\end{defn}

For simplicity, we only present the proof for the case where $\BA$ is deterministic. The randomized case can be handled similarly, by additionally using public randomness as $\BA$'s private randomness.

\begin{lemma} \label{lem:sq-reduction}
Let $C$ be any concept class, and suppose that there exists an algorithm $\BA$ that can $\alpha$-learn $C$ using $q$ $\stat_{\MD}(\tau)$ queries.  Then, there exists a learner $\BA'$ that is $m$-sample $(\alpha, \beta)$-accurate $(1 - \beta, 1 - \beta)$-pseudo-globally stable, where $m =  \tilde{O}_{\beta}\left(q^3 / \tau^2\right)$. Furthermore, the running time of $\BA'$ is at most $\poly(\TIME(\BA), 1 / \tau, 1/\beta)$.
\end{lemma}

Plugging~\Cref{lem:sq-reduction} into~\Cref{thm:stability-to-apx-dp,thm:stability-to-pure-dp,thm:stability-to-local-dp} implies \Cref{thm:sq-reduction}.

\begin{proof}[Proof of Lemma~\ref{lem:sq-reduction}]

Let $I = \lceil 3/\tau \rceil$, $\Omega = \{0, 1/I, \dots, (I - 1)/I, 1\}$ and $m' = 10^6 \cdot \frac{q^2}{\beta^2 \tau^2} \cdot \log\left(\frac{q^2}{\beta^4 \tau^2}\right)$.
Let $\CS$ with randomness $\MR'$ be the correlated sampling strategy for $\Omega$ given in \Cref{thm:correlated-sampling}. Let $\MR = (\MR')^{\otimes q}$. 

For $r = (r'_1, \dots, r'_q) \sim \MR$, we simulate the oracle $\stat_{\MD}(\tau)$ as follows:

\begin{algorithm}[ht]
\caption{Pseudo-Globally Stable Learner from SQ.}
\begin{algorithmic}
\FOR {the $i$th query $\phi_i$ to the SQ oracle}
\STATE Draw $m'$ samples $S_i \sim \MD^{m'}$
\STATE $\hu_i \gets \frac{1}{m's} \sum_{z \in S_i} \phi_i(z)$
\STATE Let $\hat{\MP}_i$ denote the probability distribution on $\Omega$ where
\begin{align*}
\hat{\MP}_i(\ell / I) =
\begin{cases}
\hu_i - \lfloor \hu_i \rfloor & \text{if } \ell = \lfloor \hu_i \rfloor  + 1, \\
1 - (\hu_i - \lfloor \hu_i \rfloor) & \text{if } \ell = \lfloor \hu_i \rfloor, \\
0 & \text{otherwise.}
\end{cases}
\end{align*}
\STATE Output $\CS(\hat{\MP}_i; r'_i)$
\ENDFOR
\end{algorithmic}
\end{algorithm}

Let $\BA'$ denote the learner that runs $\BA$ using the above oracle simulation. It worth keeping in mind that, when $\BA$ is adaptive, $\hat{\MP}_i$ depends on all of $S_1, r'_1, \dots, S_{i-1}, r'_{i - 1}, S_i$; however, we do note write this explicitly for notational ease.

Now, consider the ``distributional'' runs of the algorithm $\BA$ where the oracle is as defined above except the ``empirical'' $\hu_i$ is replaced by the ``distributional'' $u_i := \E_{z \sim \MD}[\phi_i(z)]$, and similarly $\hat{P}_i$ is replaced by $\MP_i$ 
where
\begin{align*}
\MP_i(\ell / I) =
\begin{cases}
u_i - \lfloor u_i \rfloor & \text{if } \ell = \lfloor u_i \rfloor  + 1, \\
1 - (u_i - \lfloor u_i \rfloor) & \text{if } \ell = \lfloor u_i \rfloor, \\
0 & \text{otherwise,}
\end{cases}
\end{align*}
and the output of the ``distributional'' oracle is now $\CS(\MP_i; r'_i)$. Let $h_r$ denote the output of the learner $\BA$ when using this ``distributional'' oracle. Notice that these outputs depend only on the distribution $\MD$ and the randomness $r$.

Let $\mathscr{E}_i$ denote the event that the output of oracle for the $i$th query is the same in the two cases. We will show that
\begin{align} \label{eq:one-sq-query-analysis}
\Pr_{S_i, r'_i}[\mathscr{E}_i \mid \mathscr{E}_{i - 1}, \dots, \mathscr{E}_1] \geq 1 - \beta^2 / q.
\end{align}
Before we prove~\eqref{eq:one-sq-query-analysis}, let us first explain why it implies our proof. By a union bound,~\eqref{eq:one-sq-query-analysis} implies that, with probability $1 - \beta^2$ (over $S_1, \dots, S_q, r$), the ``empirical'' version of the oracle answers the same $q$ queries as the ``distributional'' version, meaning that $\MA$ will return the same output in the former as in the latter; more formally, we have
\begin{align*}
\Pr_{r \sim \MR, S_1, \dots, S_m} [\BA'(S_1, \dots, S_m; r) = h_r] \geq 1 - \beta^2.
\end{align*}
Employing Markov's inequality, we have
\begin{align*}
\Pr_{r \sim \MR} \left[\Pr_{S_1, \dots, S_m}[\BA'(S_1, \dots, S_m; r) = h_r] \geq 1 - \beta\right] \geq 1 - \beta.
\end{align*}
In other words, the algorithm is $(1 - \beta, 1 - \beta)$-pseudo-globally stable as desired. Furthermore, observe that the answer of the distributional oracle is always within $\nu$ of the true answer; as a result, the accuracy guarantee of $\BA$ also implies the accuracy of $\BA'$.

We now turn our attention to proving~\eqref{eq:one-sq-query-analysis}. Conditioned on $\mathscr{E}_{i - 1}, \dots, \mathscr{E}_1$, $\BA$ issues the same $i$th query $\phi_i$ to both the empirical and the distributional versions of the oracle. From standard concentration inequality, with probability $1 - \frac{\beta^2}{3q}$, we have $|u_i - \hat{u}_i| \leq \frac{\beta^2}{3q I}$; the latter implies $d_{\TV}(\MP_i, \hat{\MP}_i) \leq \frac{\beta^2}{3q}$. Now, using the correlated sampling guarantee (\Cref{thm:correlated-sampling}) when this occurs, we have $\Pr_{r'_i}[\CS(\MP_i; r'_i) \ne \CS(\hat{\MP}_i; r'_i)] \leq \frac{2\beta^2}{3q}$. Applying a union bound then implies~\eqref{eq:one-sq-query-analysis}.
\end{proof}

\subsection{Implications}
\label{sec:efficient-sq-examples}

While \Cref{thm:sq-reduction} may be applied to any of the many known SQ algorithms, we highlight two applications: \emph{decision lists} (cf.~\cite{Rivest87}\footnote{Note that what is commonly called decision lists (DL) today is called 1-DL by Rivest~\cite{Rivest87}, who also studied the generalization $k$-DL where each term in the list can be a $k$-conjunction. The results listed here also apply to $k$-DL but $d$ will be replaced by $d^k$.} and for definition) and \emph{linear separators over $\{0, 1\}^d$}.  Both classes are known to have efficient SQ algorithms \cite{Kearns98,DunaganV08}. As a result, we obtain:

\begin{corollary} \label{cor:sq-efficient}
There exist $\eps$-DP $(\alpha, \beta)$-accurate learners for decision lists and linear separators in the (public randomness) non-interactive local model of DP with $O\left(\frac{\poly(d) + \log(1/\beta)}{\eps^2}\right)$ users, where each user has $\poly\left(\frac{d}{\beta\alpha}\right)$ samples. Moreover, these learners run in time $\poly(d, 1/ \alpha, 1 / \beta)$.
\end{corollary}

This result is particularly interesting because our algorithm is \emph{non-interactive} meaning that the users all just send the messages to the analyzer in one round. On the other hand, Daniely and Feldman~\cite{DanielyF19} recently showed that in the item-level setting any non-interactive local DP learner requires exponential number of samples. This demonstrates the power of user-level DP learning in overcoming the non-interactivity barrier.

%\subsubsection{Concrete Example: Learning Conjunctions}

%As a concrete example, we consider $k$-conjunctions which are functions of the form $f_{i_1, \dots, i_k}: \{0, 1\}^n \to \{0, 1\}$ defined by $f_{i_1, \dots, i_k}(x) = x_{i_1} \wedge \cdots \wedge x_{i_k}$. A $k$-conjunction is also a $k$-decision list and thus \Cref{cor:sq-efficient} applies here. However, by opening up the known (non-private) algorithm, we can 

\section{Lower Bounds}

In this section we prove a lower bound (\Cref{thm:lb-main}) on the number of users required in user-level private learning. The proof is essentially identical to that of~\cite{BeimelNS19} for the item-level setting.

We will need the following well-known bound often referred to as ``group privacy'':

\begin{lemma}[Group Privacy~\cite{DworkR14}] \label{lem:group-privacy}
Let $\BA$ be any $(\eps, \delta)$-DP algorithm and $O$ be any subset of outputs. Suppose that $D, D'$ are two datasets such that we can transform one to another by a sequence of at most $k$ addition/removal of users. Then, we have
\begin{align*}
\Pr[\BA(D) \in S] \leq e^{k\eps} \cdot \Pr[\BA(D') \in S] + \frac{e^{k\eps} - 1}{e^{\eps} - 1} \cdot \delta.
\end{align*}
\end{lemma}

\begin{proof}[Proof of \Cref{thm:lb-main}]
We will prove the contrapositive. Suppose that there exists an $(\eps, \delta)$-DP $(\alpha, 1/2)$-accurate learner $\BA$ for the class $C$ that requires only $n \leq 0.01\log(1/\delta)/\eps$ users and each user has $m$ examples. We can construct a probabilistic representation $\MH$ for $C$ as follows (where $\MH$ denotes the probability of the output $H$):

\begin{algorithm}
\caption{Constructing a Probabilistic Representation.}
\begin{algorithmic}
\STATE $H \gets \emptyset$
\FOR{$T := 100 \lceil \exp(\eps n)\rceil$ times} 
\STATE Run $\BA$ on the empty set (with no users) and add its output to $H$
\ENDFOR
\STATE Output $H$
\end{algorithmic}
\end{algorithm}

From the construction, the size of $\MH$ is at most $\log T = O(\eps n)$ as desired.

To see that, $\MH$ is an $(\alpha, 1/4)$-probabilistic representation of $\MC$, consider any distribution $\MD$; let $G := \{f \mid \err{\MD}{f} \leq \alpha\}$. From the guarantee of the learner, there must be sample sets $S_1, \dots, S_m$ such that $\Pr[\BA(S_1, \dots, S_m) \in G] \geq 1/2$. Applying \Cref{lem:group-privacy}, we have
\begin{align*}
\Pr[\BA(\emptyset) \in G] \geq \frac{1}{e^{n\eps}} \cdot \left(\frac{1}{2} - \frac{e^{n\eps} - 1}{e^{\eps} - 1} \cdot \delta\right) 
&\geq \frac{10}{T}.
\end{align*}
Since we are running the learner $T$ times, the probability that at least one of them belongs to $H$ is at least 0.99. Thus, $\MH$ is an $(\alpha, 0.01)$-probabilistic representation of $C$.
\end{proof}

\ifshuffle
\section{Extension to Shuffle DP}
\label{app:shuffle}

In this section, we extend \Cref{thm:apx-dp-learner-generic} to the shuffle model of DP, as stated below.

\begin{theorem} \label{thm:apx-dp-learner-generic-shuffle}
Let $\alpha, \beta \in (0, 0.1)$, and $C$ be any concept class with finite $\Ldim(C) = d$. Then, for any $\eps, \delta \in (0, 1)$, there exists an $(\eps, \delta)$-shuffle-DP  $(\alpha, \beta)$-accurate learner for $C$ that requires $O\left(\log(1/(\beta\delta)) / \eps\right)$ users where each user has $\tilde{O}_{\beta}\left((d/\alpha)^{O(1)}\right)$ samples.
\end{theorem}

Recall that in the shuffle DP model~\cite{bittau17,erlingsson2019amplification,CheuSUZZ19}, each user can produce a set of messages. The messages from all users are then randomly permuted together before being sent to the analyzer. Our goal is only to ensure that the shuffled messages satisfy $(\eps, \delta)$-DP; when this holds, we say that the algorithm is $(\eps, \delta)$-shuffle-DP. Similar to our result in the local model, here we assume that the users have access to shared (but not necessarily secret) randomness.

The only ingredient in the proof of \Cref{thm:apx-dp-learner-generic-shuffle} that is specific to the central model is the DP selection algorithm (\Cref{lem:apx-selection}). Hence, it suffices to prove a shuffle-DP selection algorithm with a similar guarantee, stated below. %In fact, we will give an even stronger algorithm for the \emph{histogram} problem, where the setting is the same as selection except that we would like to output $\tc_u$ for every $u \in U$ and the error is defined as $\max_{u \in U} |\tc_u - c_u|$.

\begin{lemma} \label{lem:shuffle-dp-selection}
For any $\eps, \delta \in (0, 1)$ and $\beta \in (0, 0.1)$, there is an $(\eps, \delta)$-shuffle-DP $(O(\log(1/(\beta \delta)) / \eps), \beta)$-accurate algorithm for the selection problem.
\end{lemma}

We note that the above guarantee, unlike that of \Cref{lem:apx-selection}, does not have $\beta = 0$. However, this is just another ``bad'' event that can be included in the union bound.

In fact, our algorithm also works with for the \emph{histogram} problem, where the setting is the same as selection except that we would like to output $\tc_u$ for every $u \in U$ and the error is defined as $\max_{u \in U} |\tc_u - c_u|$. Again, we say that an algorithm histogram is $(\alpha, \beta)$-accurate if the error is at most $\alpha$ with probability at least $1 - \beta$. Here we get:
\begin{lemma} \label{lem:shuffle-dp-histogram}
For any $\eps, \delta \in (0, 1)$ and $\beta \in (0, 0.1)$, there is an $(\eps, \delta)$-shuffle-DP $(O(\log(n /(\beta \delta)) / \eps), \beta)$-accurate algorithm for the histogram problem.
\end{lemma} 

%Note that this immediately implies an $(\eps, \delta)$-shuffle-DP algorithm for the selection problem that is $(O(\log(1/\delta), \eps), 0)$-accurate: we can just run the histogram algorithm and output $\argmax_{u \in U} \tc_u$.

Selection, histogram, and related problems are well studied in the shuffle DP literature (e.g.~\cite{CheuSUZZ19,ghazi2019private,balle_merged,GKMP20-icml,ghazi2020pure,balcer2019separating,GhaziG0PV21,esa-revisited,Cheu-Zhilyaev-histogramfake,Ghazi0MPS21}). Currently, the best known algorithm for both problems yields error guarantees of either $O(\log |U| / \eps)$~\cite{ghazi2019private,balle_merged} or $O(\log(1/\delta) / \eps^2)$~\cite{balcer2019separating} for constant $\beta > 0$. Our \Cref{lem:shuffle-dp-selection} improves the latter to $O(\log(1/\delta) / \eps)$ for selection and $O(\log(n/\delta) / \eps)$ for histogram, both of which match the best known guarantees in the central model~\cite{KorolovaKMN09,BunNS19} when $\delta < 1/\poly(n)$.

\subsection{Selection Algorithm from Negative Binomial Noise}

The remainder of this section is devoted to the proof of \Cref{lem:shuffle-dp-selection} and \Cref{lem:shuffle-dp-histogram}.

\subsubsection{Binary Summation}

We start by considering an easier problem of \emph{binary summation} where each user $i$ receives an input $x_i$ and the goal is output an estimate $\ts$ of $\sum_{i \in [n]} x_i$. For this problem, we will prove the following:

\begin{lemma} \label{lem:shuffle-dp-bin-summation}
For any $\eps \in (0, 1), \delta, \beta \in (0, 0.5)$, there is an $(\eps, \delta)$-shuffle-DP for the binary summation problem such that
\begin{itemize}
\item (Underestimation) The estimate $\ts$ is always at most the true value $\sum_{i \in [n]} x_i$.
\item (Error Tail Bound) $\Pr\left[|\sum_{i \in [n]} x_i - \ts| > O\left(\log\left(\frac{1}{\delta \beta}\right) / \eps\right)\right] \leq \beta$.
\end{itemize}
\end{lemma}

To prove \Cref{lem:shuffle-dp-bin-summation}, we recall the algorithm of~\cite{GKMP20-icml-arxiv}, which adds noise drawn from the negative binomial distribution (denoted by $\NB(r, p)$) to the input; the randomizer and the analyzer\footnote{We use a slightly different analyzer than the one in~\cite{GKMP20-icml-arxiv}, where the mean of the noise was subtracted from the sum. We do not apply this step because we would like a lower bound of the true sum (\Cref{cor:nb}).} are presented in \Cref{alg:nb-randomizer} and \Cref{alg:nb-analyzer}, respectively. Their privacy guarantee, proved in~\cite{GKMP20-icml-arxiv}, is as follows\footnote{Note that \cite[Theorem 13]{GKMP20-icml-arxiv} only states the privacy guarantee for $r = 3\left(1 + \log\left(1/\delta\right)\right)$. However, since $X + Y$ where $X \sim \NB(r^2, p), Y \sim \NB(r^1 - r^2 , p)$ is distributed as $\NB(r^1, p)$ for $r^1 > r^2 > 0$, the algorithm for larger $r$ is only more private since it can be thought of as post-processing of that of a smaller $r$.}:
\begin{theorem}[{\cite[Theorem 13]{GKMP20-icml-arxiv}}] \label{thm:nb}
For any $\eps, \delta \in (0, 1)$, let $p = e^{-0.2\eps}$ and $r \geq 3\left(1 + \log\left(1/\delta\right)\right)$. Then, \Cref{alg:nb-randomizer} is $(\eps, \delta)$-shuffle-DP.
\end{theorem}

\begin{algorithm}[H]
\caption{Negative Binomial Randomizer (User $i$).} \label{alg:nb-randomizer}
\begin{algorithmic}[1]
%\Procedure{Randomizer$_{r, p, n}(x)$}{}
\STATE Samples $Z \sim \NB(r / n, p)$ 
\STATE Send $x_i + Z$ messages, where each message is 1
\end{algorithmic}
\end{algorithm}

\begin{algorithm}[H]
\caption{Negative Binomial Analyzer.} \label{alg:nb-analyzer}
\begin{algorithmic}[1]
%\PROCEDURE{Analyzer$_{r, p}$}{}
\RETURN the total number of messages received
%\ENDPROCEDURE
\end{algorithmic}
\end{algorithm}

This almost immediately implies the following corollary.

\begin{corollary} \label{cor:nb}
For any $\eps, \delta \in (0, 1)$, there is an $(\eps, \delta)$-shuffle-DP for the binary summation problem such that
\begin{itemize}
\item (Overestimation) The estimate $\ts$ is always \emph{at least} the true value $\sum_{i \in [n]} x_i$.
\item (Error Tail Bound) $\Pr\left[|\sum_{i \in [n]} x_i - \ts| > O\left(\log\left(\frac{1}{\delta \beta}\right) / \eps\right)\right] \leq \beta$.
\end{itemize}
\end{corollary}

\begin{proof}
We simply use the Negative Binomial algorithm (\Cref{alg:nb-randomizer,alg:nb-analyzer}) with $p = e^{-0.2\eps}$ and $r = \lceil 1000(1 + \log(1/(\beta\delta)) \rceil$. The privacy guarantee follows from \Cref{thm:nb}. Next, we argue the utility guarantees.
\begin{itemize}
\item Notice that user $i$ sends at least $x_i$ messages. Thus, the total number of messages received by the analyzer is at least $\sum_{i \in [n]} x_i$.
\item Since the summation of $n$ i.i.d. random variables drawn from $\NB(r / n, p)$ is distributed as $\NB(r, p)$, the total number of messages is $\sum_{i \in [n]} x_i + Y$ where $Y \sim \NB(r, p)$. Hence, to prove the error tail bound, it suffices to show that $\Pr_{Y \sim \NB(r, p)}\left[Y > O\left(\log\left(\frac{1}{\delta \beta}\right) / \eps\right)\right] \leq \beta$. To see that this is true, we follow the approach in~\cite{NB-concen}. First, observe from the definition of the negative binomial distribution that, for any $T \in \BN$, we have
\begin{align*}
\Pr_{Y \sim \NB(r, p)}\left[Y > T\right] = \Pr_{V \sim \Bin(T, 1 - p)}[V < r],
\end{align*}
where $\Bin(\cdot,\cdot)$ denotes the binomial distribution. Now, we may select $T = 2r / (1 - p) = O\left(\log\left(\frac{1}{\delta \beta}\right) / \eps\right)$. By applying the Chernoff bound, we have
\begin{align*}
\Pr_{V \sim \Bin(T, 1 - p)}[V < r] \leq \exp(-0.001 r) \leq \beta, 
\end{align*}
which, as discussed above, implies the desired error tail bound.
\qedhere

\end{itemize}
\end{proof}

Note that \Cref{cor:nb} is not yet the same as \Cref{lem:shuffle-dp-bin-summation} because the first guarantee is that the estimate is \emph{at least} in the former instead of \emph{at most} in the latter. However, it is simple to go from one to another, as formalized below.

\begin{proof}[Proof of \Cref{lem:shuffle-dp-bin-summation}]
Run the algorithm from \Cref{cor:nb} but on input $1 - x_i$ (instead of $x_i$). Let $\ts$ denote the estimate of $\sum_{i \in [n]} (1 - x_i)$ from the algorithm. We then output $n - \ts$. The accuracy and privacy guarantees follow in a straightforward manner from that of \Cref{cor:nb}.
\end{proof} 

\subsubsection{From Binary Summation to Selection}

Now that we have proved \Cref{lem:shuffle-dp-bin-summation}, let us note that it easily implies \Cref{lem:shuffle-dp-selection} by running the binary summation protocol ``bucket-by-bucket'' as described below.

\begin{proof}[Proof of \Cref{lem:shuffle-dp-selection}]
The selection algorithm works by running the $(\eps/2, \delta/2)$-shuffle-DP binary summation algorithm for each bucket in parallel\footnote{Note that, since all messages are shuffled together, we have to append $u$ to the beginning of the messages to indicate that the message corresponds to bucket $u$; see e.g.,~\cite[Appendix B]{GKMP20-icml} for a more detailed description.} where in bucket $u \in U$, we let $x_i^u = \bone[u_i = u]$. Let $\tc_u$ denote the output estimate of $\sum_{i \in [n]} x^i_u = c_u$ of bucket $u$. We finally output $u^* = \argmax_{u \in U} \tc_u$.

The privacy guarantee of the algorithm follows from the fact that changing a single $u_i$ effects at most two buckets; thus, the basic composition implies that the above algorithm is $(\eps, \delta)$-shuffle-DP.

We now argue its accuracy guarantee. Let $u^{\Opt} = \argmax_{u \in U} c_u$. From the second guarantee of \Cref{lem:shuffle-dp-bin-summation}, we have $\tc_{u^{\Opt}} \geq c_{u^{\Opt}} -O\left(\log\left(\frac{1}{\delta \beta}\right) / \eps\right)$ with probability $1 - \beta$. When this holds, the first property of \Cref{lem:shuffle-dp-bin-summation} ensures that we output $u^*$ such that
\begin{align*}
c_{u^*} \geq \tc_{u^*} \geq \tc_{u^{\Opt}} \geq c_{u^{\Opt}} - O\left(\log\left(\frac{1}{\delta \beta}\right) / \eps\right).
\end{align*}
This means that the algorithm is $\left(O\left(\log\left(\frac{1}{\delta \beta}\right) / \eps\right), \beta\right)$-accurate as desired.
\end{proof}

\subsubsection{From Binary Summation to Histogram}

The histogram protocol is similar to above except that we use the failure probability $\beta/n$ and we output zero instead of negative estimates. (This is similar to the protocols in the central model.)

\begin{proof}[Proof of \Cref{lem:shuffle-dp-histogram}]
Similar to the proof of \Cref{lem:shuffle-dp-histogram}, the histogram algorithm runs the $(\eps/2, \delta/2)$-shuffle-DP binary summation algorithm with failure probability $\beta / n$ for each bucket in parallel where in bucket $u \in U$, we let $x_i^u = \bone[u_i = u]$. Let $\tc_u$ denote the output estimate of $\sum_{i \in [n]} x^i_u = c_u$ of bucket $u$. Finally, for any bucket such that $\tc_u < 0$, we let $\tc_u = 0$ instead.

The privacy guarantee holds due to the same reason as in the proof of \Cref{lem:shuffle-dp-histogram}.

We now argue its accuracy guarantee. Let us divide the buckets into two types: $U_{> 0} := \{u \in U \mid c_u > 0\}$ and $U_{= 0} := \{u \in U \mid c_u = 0\}$.
\begin{itemize}
\item For any $u \in U_{= 0}$, due to the first property of \Cref{lem:shuffle-dp-bin-summation}, we always output zero and thus the error here is zero.
\item Now consider any $u \in U_{> 0}$. From our choice of parameters and the second property of \Cref{lem:shuffle-dp-bin-summation}, we have $\Pr\left[|c_u - \tc_u| > O\left(\log\left(\frac{n}{\delta \beta}\right) / \eps\right)\right] \leq \beta/n$. 
\end{itemize}
Note also that $|U_{> 0}| \leq n$. As a result, we may apply a union bound over all $u \in U_{ >0}$ and conclude that the error is at most $O\left(\log\left(\frac{n}{\delta \beta}\right) / \eps\right)$ with probability $1 - \beta$.
\end{proof}

\fi

\end{document}